\documentclass[10pt,journal,compsoc]{IEEEtran}
\usepackage{url}
\usepackage[noend]{algpseudocode}
\usepackage{algorithmicx,algorithm}
\usepackage{stfloats}
\usepackage{times}
\usepackage{amsmath}
\usepackage{color}
\usepackage{subfigure}
\usepackage{bm}
\usepackage{subfigure}
\usepackage{multirow}
\usepackage{booktabs}
\usepackage{bigstrut}

\usepackage[table]{xcolor}
\usepackage{amssymb}
\usepackage[justification=centering]{caption}
\usepackage{graphicx}

\usepackage{amsthm}
\usepackage[backref]{hyperref}
\makeatother
\newcommand{\vx}{\bm{x}}
\newcommand{\vV}{\bm{V}}

\newcommand{\vX}{\bm{X}}
\newcommand{\vL}{\bm{L}}
\newcommand{\vD}{\bm{D}}

\newcommand{\vS}{\bm{S}}

\newcommand{\cT}{\mathcal{T}}

\newcommand{\cL}{\mathcal{L}}
\newcommand{\vI}{\bm{I}}
\newcommand{\vv}{\bm{v}}
\newcommand{\vh}{\bm{h}}
\newcommand{\mH}{\bm{H}}

\begin{document}
	
	\title{Uniform tensor clustering by jointly exploring sample affinities of various orders}
	
	\author{Hongmin Cai$^{1\ast}$, Fei~Qi$^{1,2}$, Junyu~Li$^{1}$, Yu~Hu$^{1}$, Yue~Zhang$^{3}$,Yiu-ming Cheung$^{5}$, and Bin~Hu$^{4\ast}$
	\thanks{$^{1}$School of Computer Science and Engineering, South China University of Technology, Guangzhou, 510006, China. $^{2}$School of Data Science and Information Engineering, Guizhou Minzu University, China. $^{3}$ School of Computer Science, Guangdong Polytechnic Normal University, Guangzhou 510665, China. $^{4}$ School of Medical Technology, Beijing Institute of Technology, Beijing, China. $^{4}$ Department of Computer Science, Hong Kong Baptist University, Hong Kong SAR, China  $^\ast$Corresponding Author, E-mail: hmcai@scut.edu.cn, bh@bit.edu.cn.}
}

	\IEEEtitleabstractindextext{%
		\begin{abstract}
			Conventional clustering methods based on pairwise affinity usually suffer from the concentration effect while processing huge dimensional features yet low sample sizes data, resulting in inaccuracy to encode the sample proximity and suboptimal performance in clustering.
			To address this issue, we propose a unified tensor clustering method (UTC) that characterizes sample proximity using multiple samples' affinity, thereby supplementing rich spatial sample distributions to boost clustering. Specifically, we find that the triadic tensor affinity can be constructed via the Khari-Rao product of two affinity matrices.
			Furthermore, our early work shows that the fourth-order tensor affinity is defined by the Kronecker product. Therefore, we utilize arithmetical products, Khatri-Rao and Kronecker products, to mathematically integrate different orders of affinity into a unified tensor clustering framework. Thus, the UTC jointly learns a joint low-dimensional embedding to combine various orders. Finally, a numerical scheme is designed to solve the problem. Experiments on synthetic datasets and real-world datasets demonstrate that 1) the usage of high-order tensor affinity could provide a supplementary characterization of sample proximity to the popular affinity matrix; 2) the proposed method of UTC is affirmed to enhance clustering by exploiting different order affinities when processing high-dimensional data.
		\end{abstract}
		
		}
	
	\maketitle
	
	\IEEEpeerreviewmaketitle
	
	\IEEEraisesectionheading{\section{Introduction}\label{sec:introduction}}
	\IEEEPARstart{C}{lustering} aims to partition unlabeled data into distinct groups~\cite{ref1}. Traditional methods, such as K-means~\cite{ref3} and spectral clustering~\cite{Ncut,RatioCut}, have been well studied, and many variants have been designed to cater to various needs. However, with advances in observatory equipment, an object of interest could be associated with high or even huge dimensional features. For example, in bioinformatics, most single cells isolated from tissue are damaged during sequencing. Thus, only a limited number of cells can be successfully sequenced, resulting in a sample matrix with a few samples and millions of gene reads~\cite{li2020modern,tian2019clustering}. Similar scenarios are rising across various fields, such as computer vision~\cite{9000790,8725928} and natural language processing.   Accurate clustering  high-dimension $m$ yet low-sample-size $n$ (HDLSS) data remains challenging when  $n \gg m$ ~\cite{yang2018clustering,chen2019knn,xu2018accelerated}.
	
	The critical challenge that hinders the clustering of HDLSS data is  called the concentration effect~\cite{aggarwal2001surprising,franccois2007concentration}, which refers to the situation that the pairwise Euclidean distance among the samples collapses to a constant, thus resulting in the sample-wise affinity tending to be indiscriminating in high feature dimensions~\cite{SarkarGhosh-350,BorysovHannig-352,HallMarron-351}. Such an effect exists as an insurmountable roadblock that hinders most clustering methods, which rely on  the pairwise affinity of samples, from achieving precise clustering. Early works attempted to learn adaptive sample affinity to develop metrics driven by data.  For example, CAN~\cite{CAN} proposed a model that dynamically learns the affinity matrix by assigning the adaptive neighbors for each data point based on  local distances. However, the measurement of local distance remains a characterization of sample-wise affinity and thus still suffers from the concentration effects and performs unsatisfactorily.

	To overcome such drawbacks, many researchers have proposed to incorporate the spatial distribution of multiple samples to enhance clustering. Such high-order characterization on the local structure of multiple points can effectively overcome the shortcomings of pairwise similarity and enhance the clustering effect~\cite{ref12}~\cite{ref13}. For example, \cite{SpectralHypergraphE} generalizes the spectral methods operating on a pairwise similarity graph to a hypergraph that can represent high-order affinities and develops an algorithm to obtain the hypergraph embedding for clustering. \cite{CAGE} proposes to maintain a weight graph to approximate the hypergraph generated by high-order association and then use a spectral method to accomplish graph partitioning. \cite{ref11} considers the hypergraph clustering problem as a multiplayer noncooperative game from a game-theoretic approach. Like the CAN method, \cite{DHSL} first proposes a way to learn the dynamic hypergraph for clustering. However, the above method must eventually convert the hypergraph into an approximate weighted graph. Therefore, the clustering task is still performed on sample affinities, so higher-order affinity is not properly exploited to overcome the drawbacks of pairwise relations.

	Recently, the use of tensors to encode high-order affinities has been studied in clustering. \cite{Zeigen} introduce the characteristic tensor of a hypergraph and its associated $Z$-eigenvalue, providing natural relations for the structure of data. \cite{AAAI} denotes the order-$k$ affinities by a symmetric tensor and relaxes the hypergraph clustering to solve the multilinear singular value decomposition problem. However, these methods can use only a specific order of affinity for clustering, which cannot completely express the data structure. Our early work of IPS2~\cite{ref15} concentrates on using the high-order structural information of data that can be obtained from relationships between two pairs to enhance clustering performance. The authors establish an equivalence between the even-order affinity and the pairwise second-order affinity.
	However, there are still two problems that need to be solved: (1) the connection between odd-order tensor affinity and the second-order matrix affinity remains unknown; (2) how to jointly utilize affinities with different orders to accomplish  clustering has yet to be studied.
	
	To address these two questions, this paper first establishes a connection between the second-order and third-order tensor affinities via the Khatri-Rao product. Then, we propose a unified tensor clustering (UTC) model to learn a low-dimensional embedding by jointly using affinities of various orders. The popular second-order matrix affinity and  the triadic and  tetradic tensor affinities are fused to help characterize the proximity among samples.
	
	The main advantages of the method proposed in this paper are summarized as follow.
	\begin{itemize}
		\item[1)] A connection between pairwise matrix affinity and the triadic tensor affinity among three samples is established. An undecomposable third-order tensor is designed to provide the geometric proximity among three samples, thus serving as a supplement to the pairwise similarity matrix.
		\item[2)] The third-order tensor affinity is defined  via the Khatri-Rao product, while the Kronecker product is used to define the fourth-order tensor affinity. Therefore, the popular matrix affinity defined by the arithmetical product is surprisingly related to the Khatri-Rao and Kronecker products.
		\item[3)] A uniform tensor clustering method (UTC) is developed. The UTC learns a  low-dimensional embedding jointly based on affinities of various orders. The UTC is formulated elegantly and works on affinities of any order.
		\item[4)] Extensive experiments on HDLSS data demonstrate that the UTC achieved superior performance compared with that of baseline methods. In addition, the experiments show that using high-order affinities to characterize the sample spatial distribution can improve clustering performance, especially in cases of small sample sizes.
		
	\end{itemize}

	\begin{figure*}[ht]
		\centering
		\includegraphics[width = 17cm]{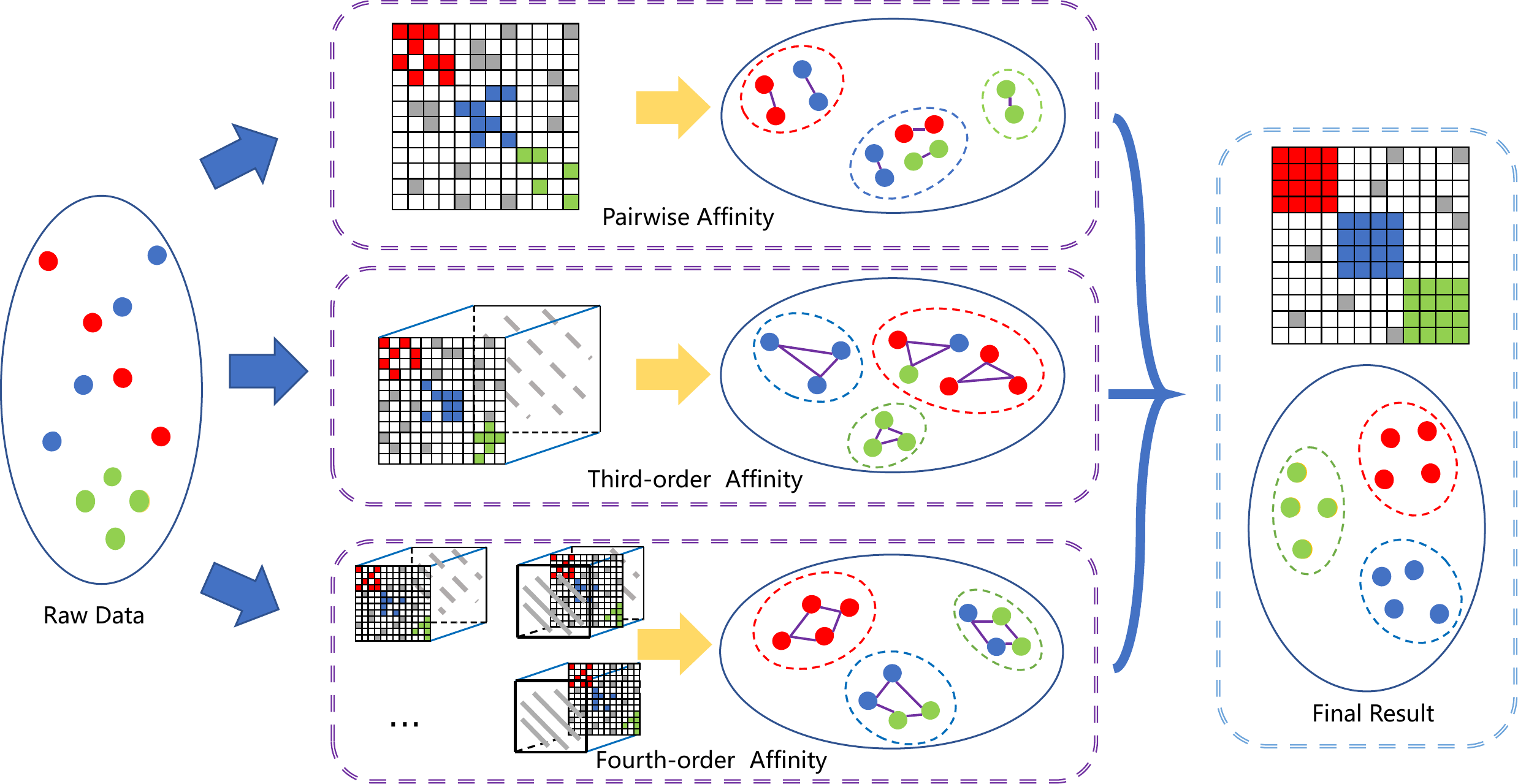}
		\caption{Demonstration of the UTC method, which aims to learn a uniform embedding via joint  optimization of   not only the pairwise matrix affinity  but also the triadic and tetradic tensor affinities among  more than two samples. The approach is combined with  the high-order tensor affinity to provide complementary spatial characterization for which the  pairwise affinity fails;   UTC learns a low-dimensional embedding and is robust to high-dimensional concentration effects and remains stable.}
		\label{fig1_farmework}
	\end{figure*}

	\section{NOTATIONS AND PRELIMINARIES}
	\subsection{Notations}
	Throughout the paper,  lowercase letters represent scalars, while bold lowercase letters represent vectors, such as $ \bm v \in \mathbb R^m$. Bold uppercase letters denote matrices, such as $ \bm M \in \mathbb R^{m\times n}$, while bold calligraphy letters represent tensors, such as, $ \cT \in \mathbb R^{ m \times n\times l}$.   Let $\cT(:,:,i)$, $\cT(:,i,:)$, $\cT(i,:,:)$ represent the $i$-th frontal, lateral and horizontal slices of an order-3 tensor $\mathcal T \in \mathbb R^{n_1\times n_2 \times n_3}$, respectively. $\cT(:,:,i)$ can be abbreviated as $\cT^{(i)}$. A tensor can be reformatted  into a matrix via series operations, called {\it unfolding}. For example, the three-order and four-order tensor unfold operations are as follows:
	
	\newtheorem{definition}{Definition}[section]
	\begin{definition}
		\textbf{Unfolding 3-order Tensor:} Let  $\cT_3 \in \mathbb R^{m\times m \times m}$ be an order-3 m-dimensional tensor. It can unfold to an $m^2\times m$ matrix $\hat {\bm T}_3$ as follows:
		\begin{equation}
			\begin{aligned}
				\hat {\bm T_3}= \operatorname{unfold}(\mathcal{T}_3)=\left[
				\begin{array}{c}
					\mathcal{T}_3^{(1)} \\
					\mathcal{T}_3^{(2)} \\
					\vdots \\
					\mathcal{T}_3^{\left(m\right)}
					\end
					{array}\right]
				\end{aligned}
				\label{unfolding_order_3}
			\end{equation}
		\end{definition}

		\begin{definition}
			\textbf{Unfolding 4-order Tensor} Let $\cT_4 \in \mathbb R^{m\times m \times m\times m}$ represent a fourth-order m-dimensional tensor. This tensor can unfold to an $ m^2\times m^2 $ matrix $\hat T_4$, with its (r,s)-th entry given by:
			\begin{equation}
				\begin{aligned}
					\hat {\bm T}_{4_{rs}} = \mathcal T_{4_{ijkl}}
				\end{aligned}
				\label{unfolding_order_4}
			\end{equation}
			with $ r=m(j-1)+i,s=m(l-1)+k, i,j,k,l\in [1, m] $.
		\end{definition}
		
		Three  matrix multiplications are considered, namely, Hadamard product, Kronecker product and Khatri-Rao product~\cite{def_tensor_product}.
		\begin{definition}
			\textbf{Hadamard Product}   The Hadamard product of two matrices $\bm A\in \mathbb R^{m\times n}$ and $\bm B\in \mathbb R^{m\times n}$ is defined as:
			
			\begin{equation}
				\begin{aligned}
					\bm A \odot \bm B = \left[ \begin{matrix}
						a_{11}b_{11}&   \cdots&   a_{1n}b_{1n}\\
						\vdots&   \ddots&   \vdots\\
						a_{m1}b_{m1}&   \cdots&   a_{mn}b_{mn}\\
					\end{matrix} \right] \in \mathbb R^{m\times n}
				\end{aligned}
			\end{equation}
		\end{definition}
		
		\begin{definition}
			\textbf{Kronecker Product} The Kronecker product of two matrices $\bm A\in \mathbb R^{m_1\times n_1}$ and $\bm B\in \mathbb R^{m_2\times n_2}$ is defined as:
			\begin{equation}
				\begin{aligned}
					\bm{A}\otimes \bm{B}=\left[ \begin{matrix}
						a_{11}\bm{B}&   \cdots&   a_{1n}\bm{B}\\
						\vdots&   \ddots&   \vdots\\
						a_{m1}\bm{B}&   \cdots&   a_{mn}\bm{B}\\
					\end{matrix} \right] \in \mathbb R^{m_1m_2\times n_1n_2}
				\end{aligned}
			\end{equation}
		\end{definition}
		\begin{definition}
			\textbf{Khatri-rao Product} The Khatri-Rao product of two matrices $\bm A\in \mathbb R^{m_1\times n }$ and $\bm B\in \mathbb R^{m_2\times n }$ is defined as the matrix:
			\begin{equation}
				\begin{aligned}
					\bm A *\bm B=\left[ \begin{matrix}
						a_{:1}\otimes b_{:1}&   \cdots&   a_{:n}\otimes b_{:n}
					\end{matrix} \right] \in \mathbb R^{m_1m_2\times n}
				\end{aligned}
			\end{equation}
		\end{definition}
		
		A popular product between a tensor and a matrix is called a {\it $k$-mode product}.
		\begin{definition}
			\textbf{k-mode Product}  The $k$-mode product between an order-$m$ tensor $\bm {\mathcal T} \in \mathbb R^{n_1\times n_2\times \dots \times n_m}$ and a  matrix $V \in \mathbb R^{p\times n_k}$,  denoted by $\bm {\mathcal T} \otimes_k U \in \mathbb R^{n_1\times \dots \times n_{k-1} \times p \times n_{k+1} \times \dots \times n_m}$, with
			\begin{equation}
				\begin{aligned}
					(\bm{\mathcal T}\otimes_k \vV)_{i_1\dots i_{k-1}ji_{k+1}\dots i_m} = \sum^{n_k}_{i_k=1}\bm{\mathcal T}_{i_1 \dots i_{k-1} i_{k} i_{k+1} \dots i_m} \vV_{ji_{k}}
				\end{aligned}
			\end{equation}
		\end{definition}
		
		For convenience, we summarize  the frequently used notations and definitions in Table~\ref{Notions_Definitions}.
		
		\begin{table}[]
			\caption{ Summary of notations}
			\begin{center}
				\begin{tabular}{c|l}
					\hline
					$\vv \in \mathbb R^m$                                             & A vector      \\
					$\bm M \in \mathbb R^{m\times n} $                                & A matrix       \\
					$\cT \in \mathbb R^{n_1\times n_2 \times \dots \times n_k}$& An order-k tensor \\
					$\hat{\bm T}_k $& The unfolding order-k tensor            \\
					$Tr(\cdot)$       & The trace operation                     \\
					$\odot$          & The Hadamard product                     \\
					$*$              & The Khatri-Rao product                  \\
					$\otimes$        & The Kronecker product                    \\
					$\otimes_k$       & The k-mode product                       \\\hline
				\end{tabular}
			\end{center}
			\label{Notions_Definitions}
		\end{table}

		\subsection{Introduction of Spectral Clustering with Pairwise Similarity}
		
		Spectral clustering~\cite{Ncut} starts by computing a sample-wise affinity matrix $\bm S_2 \in\mathbb R^{m\times m }$. Every entry $\bm S_2(i,j)$  measures the proximity between two samples $x_i$ and $x_j$. Then, a normalized affinity matrix $\hat{\vL}_2$ is estimated with $\hat{\vL}_2=\bm D_2^{-\frac{1}{2}}\bm S_2\bm D_2^{-\frac{1}{2}}$, where $\bm D_2$ denotes the degree matrix of  $\bm S_2$ with $\bm D_{2_{ii}} = \sum_i \bm S_{2{ij}}$. The Laplacian matrix $\bm L$ can be generated through $\bm L = \bm I - \hat{\vL}_2$. Spectral graph clustering seeks a low-dimensional embedding of samples by maximizing the cross-entropy while preserving the local proximity of paired samples:
		\begin{eqnarray}\label{eq:sp}
			\begin{aligned}
				&\min_{\vV} \,\, tr(\vV^T\vL\vV)\\
				\text{ Subject to:~} \quad &\vV^T\vV=\vI\\
			\end{aligned}
		\end{eqnarray}
		It is a standard eigenvalue problem and thus can be efficiently solved by calculating the eigenvectors $\vV$ of the Laplacian matrix $\vL$, using methods such as singular value decomposition.
		Typically, the normalized cut of the graph is minimized in this model, but as stated in~\cite{max_N-cut}, this model can also be expressed as a maximization problem with the normalized pairwise affinity $\hat{\vL}_2$. By using the tensor $k$-mode product, the spectral clustering in Eq.~(\ref{eq:sp})  can be rewritten as:
		\begin{equation}
			\begin{aligned}
				& \max_{\vV} \,\, tr(\vV^T\hat{\vL}_2\vV)\\
				=&\max_{\vV} \,\, \sum_{j=1}^k \hat{\vL}_2\otimes _2\vv_j^T\otimes _1\vv_j^T\\
				&\text{ Subject to:~}  \vV^T\vV=\bm I
			\end{aligned}
			\label{SC_Obj_2}
		\end{equation}
		where $\vv_j$ is the $j^{th}$ column of the partition matrix $\vV$. This equation establishes the relationship between the membership $\vv_j$ and the pairwise affinity. Then, a clustering task is performed on the obtained embedding.

		\subsection{Introduction of IPS2 Clustering with Pair-to-Pair Similarity}
		The pairwise affinity is easily broken by noise contamination or concentration effects in HDLSS data. To address this issue, our recent work IPS2~\cite{ref15} attempted to  use high-order affinity among more than two samples. IPS2 used the fourth-order tensor affinity to measure the proximity of two sample pairs. The decomposable fourth-order tensor affinity  can be defined as the product of two pairwise similarities:
		\begin{equation}
			\begin{aligned}
				\cT_{4_{ijkl}} = \bm S_{2_{ik}}\bm S_{2_{jl}},\quad i,j,k,l\in m.
			\end{aligned}
		\end{equation}
		
		
		\newtheorem{theorem}{Theorem}[section]

		Let $\hat{\bm T}_4$ denote the matrix unfolded from a  fourth-order decomposable tensor affinity $\mathcal T_4$; it has been shown in ~\cite{ref15} that
		\begin{eqnarray}\label{eq:similarity_equvalence}
			\hat{\bm T}_4 = \bm S_2 \otimes \bm S_2
		\end{eqnarray}
		
		This result establishes a direct connection between a  fourth-order decomposable tensor affinity and the usual second-order matrix affinity $\bm S_2$. Now, consider a normalized affinity matrix $\hat{\bm L}_4$, computed by $\hat{\bm L}_4 = \hat {\bm D}_4^{-\frac{1}{2}}\hat{\bm T}_4 \hat{\bm D}_4^{-\frac{1}{2}}$. Here, the diagonal matrix $\hat{\bm D}_4$ is the degree matrix of $\hat{\bm T}_4$ with the diagonal elements being $\hat{\bm D}_{4_{ii}} = \sum_j \hat{\bm T}_{4_{ij}}$. It has been shown that the fourth-order and  second-order  normalized affinity matrix $\hat{\bm L}_4, \hat{\vL}_2$ shares a similar  relation as in Eq.~(\ref{eq:similarity_equvalence}):

		\begin{theorem}
			\textbf{The decomposable 4-order tensor~\cite{ref15}} Let $\hat{\vL}_2$ and $\mathcal L_4$
			denote the normalized similarity matrix and fourth-order normalized affinity tensor with the unfolding form $\hat{\bm L}_4$, respectively. Then, we have  the following equality:
			
			\begin{equation}
				\begin{aligned}
					\hat{\bm L}_4 = \hat{\vL}_2 \otimes \hat{\vL}_2
				\end{aligned}
			\end{equation}
			
			Moreover, the eigenvectors $\vv$ and $\hat \vv$ for the matrix $\hat{\vL}_2$ and unfolded matrix $\hat{\bm L}_4$ satisfy:
			\begin{equation}
				\begin{aligned}
					\hat{\bm v} = \vv \otimes \vv
				\end{aligned}
			\end{equation}
			\label{thm 2.2}
		\end{theorem}
		
		One can reshape the eigenvector $\hat{\vv}$ into a matrix $\bm V\in \mathbb{R} ^{m\times m}$, and the matrix $\vV$  is an eigenmatrix of the  fourth-order tensor affinity  $ \hat{\bm T}_4$ in Eq.~(\ref{eq:similarity_equvalence}). The established equivalence between the decomposable tensor affinity $\mathcal{T}_4$ and any similarity matrix $\bm S$ opens a new path to understand high-order affinity. However, the fourth-order decomposable affinity inherits the drawback of the low-order similarity~\cite{fused_drawback} in that it is vulnerable to noise corruption and concentration effects when applied to high-dimensional samples. Alternatively, one can use undecomposable tensor affinity to provide complementary sample affinities that the pairwise matrix affinity could not provide.  The IPS2 uses a Fisher-ratio-like tensor affinity:
		\begin{equation}\label{eq:order-4Sim}
			\mathcal{T} _{ijkl}=exp(-\sigma \frac{d_{ij}+d_{kl}}{d_{ik}+d_{jl}+\varepsilon})
		\end{equation}
		for $i,j,k,l\in n$, where $d_{ij}$  denotes the distance between samples $x_i$  and $x_j$ and the parameter $\sigma$ is an scaling constant.
		
		One can calculate the normalized tensor affinity $\mathcal{L} _4$ from the fourth-order tensor affinity  $\mathcal{\bm T}_4$. Then, we can learn a low-dimensional embedding $\vV$ from the fourth-order tensor affinity by solving the following maximization problem:
		\begin{equation}
			\begin{aligned}
				\max_{\vV} &\sum_{j=1}^k \mathcal{L} _4\otimes _4\vv_j\otimes _3\vv_j\otimes _2\vv_j\otimes _1\vv_j\\
				&=tr((\vV * \vV)^T\hat{\bm L}_4(\vV * \vV))\\
				& s.t. \vV^T\vV = \bm I
			\end{aligned}
			\label{SC_Obj_L24}
		\end{equation}
		\section{Method}
		The performance of clustering based on pairwise affinity suffers in high-dimensional and small-sample-size datasets. Incorporating the spatial distribution among multiple samples instead of using only pairwise affinity improves the clustering performance. IPS2~\cite{ref15} built a bridge between the fourth-order affinity and low-dimensional embedding. It has been shown that high-order information can effectively complete the description of spatial data structure from low-order information. However, it is restricted to using even-order affinity, such as the fourth-order affinity, and does not apply to odd-order affinity. Moreover, it fails to fuse affinities from different orders within a uniform formulation, resulting in cumbersome utility for empirical applications.
		To address these two issues, we did the following work: (1) we introduced triadic and tetradic tensor affinity for three and four samples;
		(2) we propose a uniform framework to learn a low-dimensional embedding by fusing affinity of various orders. An illustrative overview of UTC is shown in Fig.~\ref{fig1_farmework}.

		\subsection{Characterizing Sample's Affinity via Third-Order Tensor Affinity}

		\subsubsection{Decomposable Third-Order Tensor Affinity}
		Popular clustering methods are based on the pairwise similarity of two samples. We are interested in the tensor affinity   among three samples, denoted by a third-order tensor $\mathcal T_3= [\cT_{ijk}]\in \mathbb{R} ^{m\times m \times m}$. A natural way to define such affinity is to use the composite affinity based on paired affinities, such that each entry $\cT_{ijk}$ is given by:
		\begin{equation}\label{eq:third_order_def}
			\begin{aligned}
				\cT_{3_{ijk}} = \bm S_{2_{ij}} \bm S_{2_{kj}}
			\end{aligned}
		\end{equation}
		Under this definition, one can generalize the pairwise affinity $\bm S_2$ to a third-order tensor affinity  $\mathcal T_3$. Specifically, the results satisfy the following properties.
		
		\begin{theorem}
			Given a decomposable third-order tensor affinity $\mathcal T_3$ defined in Eq.~(\ref{eq:third_order_def}), $\hat{\bm T}_3 = \bm S_2 * \bm S_2$, where  $\hat{\bm T}_3$ is the matrix unfolded  by the tensor $\cT_3$.
			\label{thm 3.1}
		\end{theorem}
		
		\begin{proof}
			By the definition in  Eq.(~\ref{unfolding_order_3}), the unfolded matrix  $\hat{\bm T_3} \in \mathbb R^{n^2\times n}$ can be written:
			\begin{equation}
				\begin{aligned}
					\bm{\hat{T}}_3 &=\left[ \begin{matrix}
						\mathcal{T} _{3_{111}}&   \dots&    \mathcal{T} _{3_{1n1}}\\
						\vdots&   \ddots&   \vdots\\
						\mathcal{T} _{3_{n1n}}&   \dots&    \mathcal{T} _{3_{nnn}}\\
					\end{matrix} \right]
					\\
					&=\left[ \begin{matrix}
						\bm S_{2_{11}}\bm S_{2_{11}}&   \dots&    \bm S_{2_{1n}}\bm S_{2_{1n}}\\
						\vdots&   \ddots&   \vdots\\
						\bm S_{2_{11}}\bm S_{2_{n1}}&   \dots&    \bm S_{2_{1n}}\bm S_{2_{nn}}\\
						\vdots&   \ddots&   \vdots\\
						\bm S_{2_{n1}}\bm S_{2_{11}}&   \dots&    \bm S_{2_{nn}}\bm S_{2_{1n}}\\
						\vdots&   \ddots&   \vdots\\
						\bm S_{2_{n1}}\bm S_{2_{n1}}&   \dots&    \bm S_{2_{nn}}\bm S_{2_{nn}}\\
					\end{matrix} \right]
					\\
					&=\left[ \begin{matrix}
						\bm S_{2_{11}}\bm S_{2_{:1}}&   \dots&    \bm S_{2_{1n}}\bm S_{2_{:n}}\\
						\vdots&   \ddots&   \vdots\\
						\bm S_{2_{n1}}\bm S_{2_{:1}}&   \dots&    \bm S_{2_{nn}}\bm S_{2_{:n}}\\
					\end{matrix} \right]
					\\
					&=\left[ \begin{matrix}
						\bm S_{2_{:1}} \otimes \bm S_{2_{:1}}&  \dots&  \bm S_{2_{:n}}\otimes \bm S_{2_{:n}}\\
					\end{matrix} \right]
					\\
					&= \bm S_2 * \bm S_2
				\end{aligned}
			\end{equation}
		\end{proof}
		
		Surprisingly, this theorem shows that a decomposable third-order tensor affinity can be obtained via the Khatri-Rao product of two affinity matrices. 
		We now proceed to define a normalized affinity tensor $\cL_3$. Let a matrix $\hat {\bm L}_3$ be calculated by $\hat {\bm L}_3 = \hat {\bm D}_{3_1}^{-\frac{1}{2}} \hat {\bm T}_3  \hat {\bm D}_{3_2}^{-\frac{1}{2}}$. Here, the two diagonal matrices $\hat{\vD}_{3_1}, \hat{\vD}_{3_2}$ are given by $\hat {\bm D}_{3_1} =\sqrt{\sum_i{\hat{\bm T}_{3_{:i}}}\sum_i{\hat{\bm T}_{3_{:i}}}} =\bm D_2\otimes \bm D_2$ and $\hat {\bm D}_{3_2} = \sum_i{\hat{\bm T}_{3_{:i}}} =\bm D_2 \odot \bm D_2$. The normalized third-order affinity tensor $\cL_3$ is obtained by folding the matrix $\hat {\bm L}_3$.
		
		This defined affinity tensor  shares a similar property to that of the fourth-order affinity tensor.
		
		\begin{theorem}
			Let a matrix $\bm S_2$ be a pairwise affinity matrix with $\hat{\vL}_2$ being its normalized affinity matrix. The decomposable third-order tensor affinity  $\cT_3$  and its normalized tensor $\cL_3$ are obtained as above. We have
			\begin{equation}
				\begin{aligned}
					\hat{\bm{L}}_3 = \hat{\vL}_2 * \hat{\vL}_2
				\end{aligned}
			\end{equation}
			where $\hat{\bm L}_3$ is the unfolded matrix from the tensor $\cL_3$.
			
			\label{thm 3.2}
		\end{theorem}
		
		\begin{proof} By the definitions of the normalized affinity matrix $\hat{\vL}_2$ and $\hat{\bm L}_3$, we have:
			\begin{equation}
				\begin{aligned}
					\hat{\bm L}_3 =&\hat{\bm D}_{3_1}^{-\frac{1}{2}}\hat{\bm T_3}\hat{\bm D}_{3_2}^{-\frac{1}{2}}\\
					=&(\bm D_2\otimes \bm D_2)^{-\frac{1}{2}}(\bm S_2* \bm S_2)(\bm D_2\odot \bm D_2)^{-\frac{1}{2}}\\
					=&(\bm D_{2}^{-\frac{1}{2}}\bm S_2\bm D_{2}^{-\frac{1}{2}})*(\bm D_{2}^{-\frac{1}{2}}\bm S_2\bm D_{2}^{-\frac{1}{2}})\\ =&\hat{\vL}_2*\hat{\vL}_2\\
				\end{aligned}
			\end{equation}
		\end{proof}
		 Theo.~\ref{thm 3.1} and Theo.~\ref{thm 3.2} build a  connection between third-order affinity $ \cT_3 $  and pairwise similarity $\bm S_2$ via the Khatri-Rao product. We construct a synthetic dataset consisting of eighty samples, evenly drawn from two different clusters. Then, we visualize the pairwise affinity $\vS_2$ and the decomposable third-order and fourth-order tensor affinity in  Fig.~\ref{fig3_diagG}. One can observe that they preserve the sample proximity satisfactorily. They can be well segmented into four distinct blocks, and within each sub-block, the sample proximity is small.

		\begin{figure*}[]
			\centering
			\subfigure[]{
				\includegraphics[width=5.5cm]{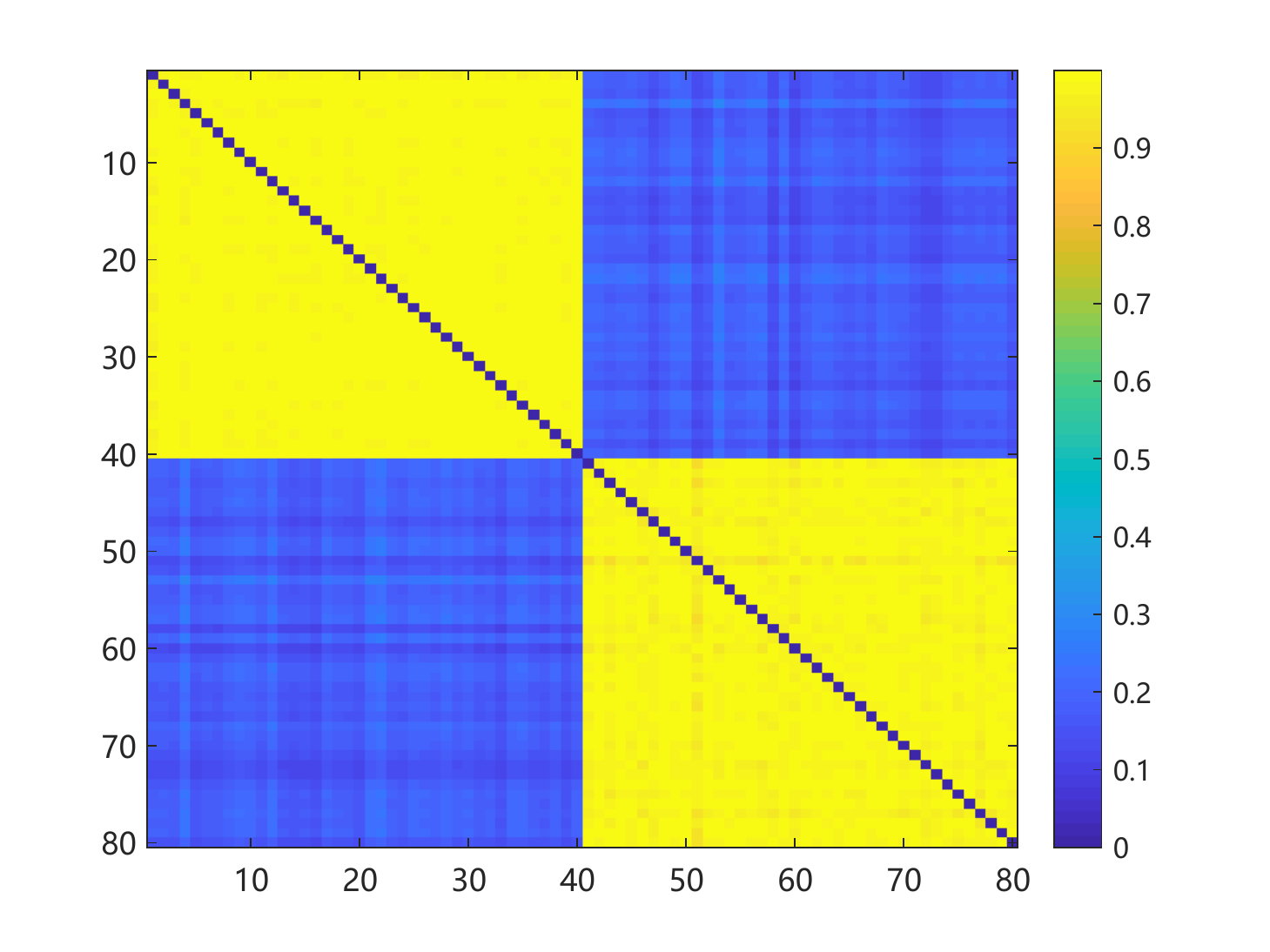}}
			\hspace{0in}
			\subfigure[]{
				\includegraphics[width=5.5cm]{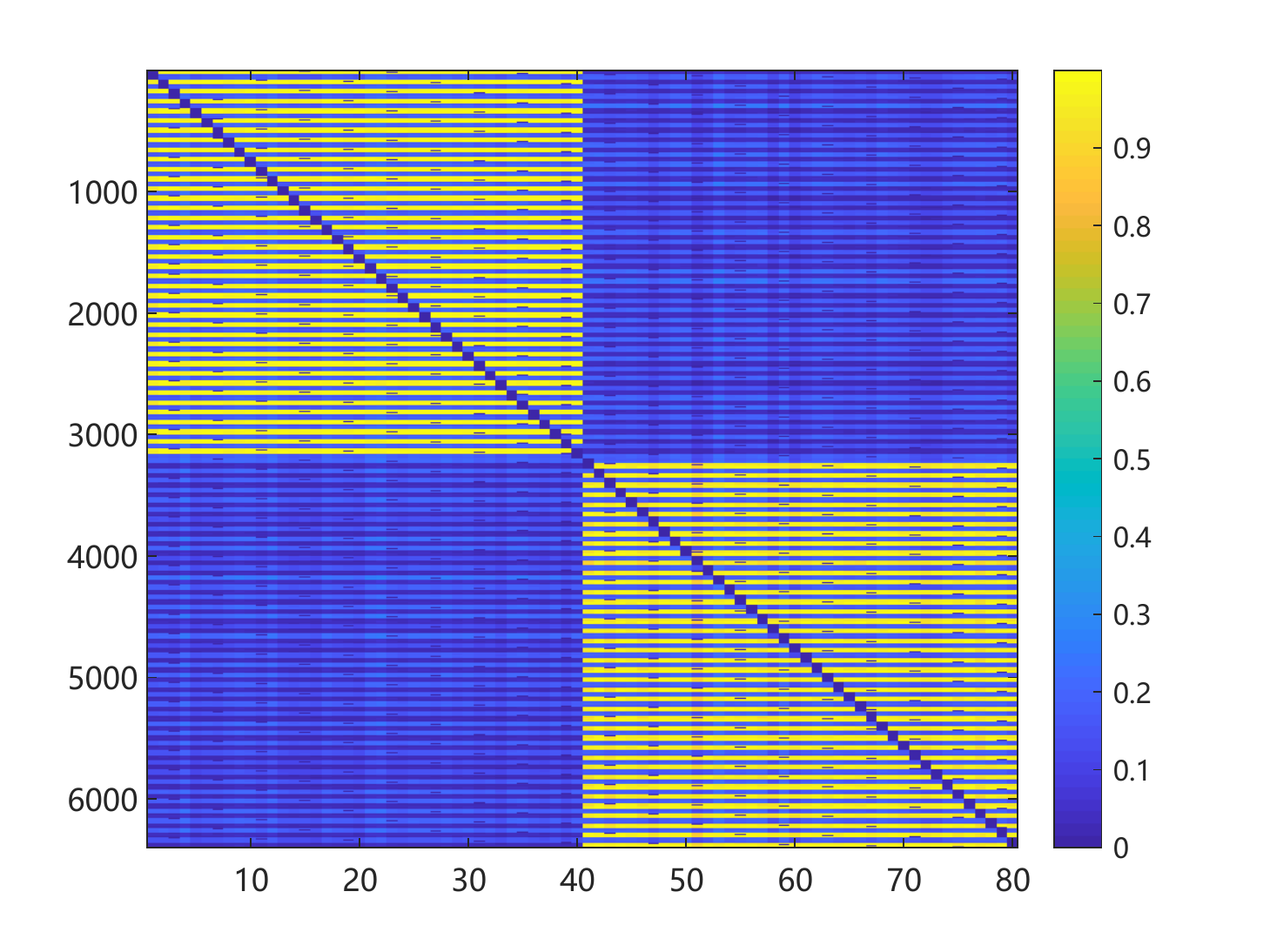}}
			\hspace{0in}
			\subfigure[]{
				\includegraphics[width=5.5cm]{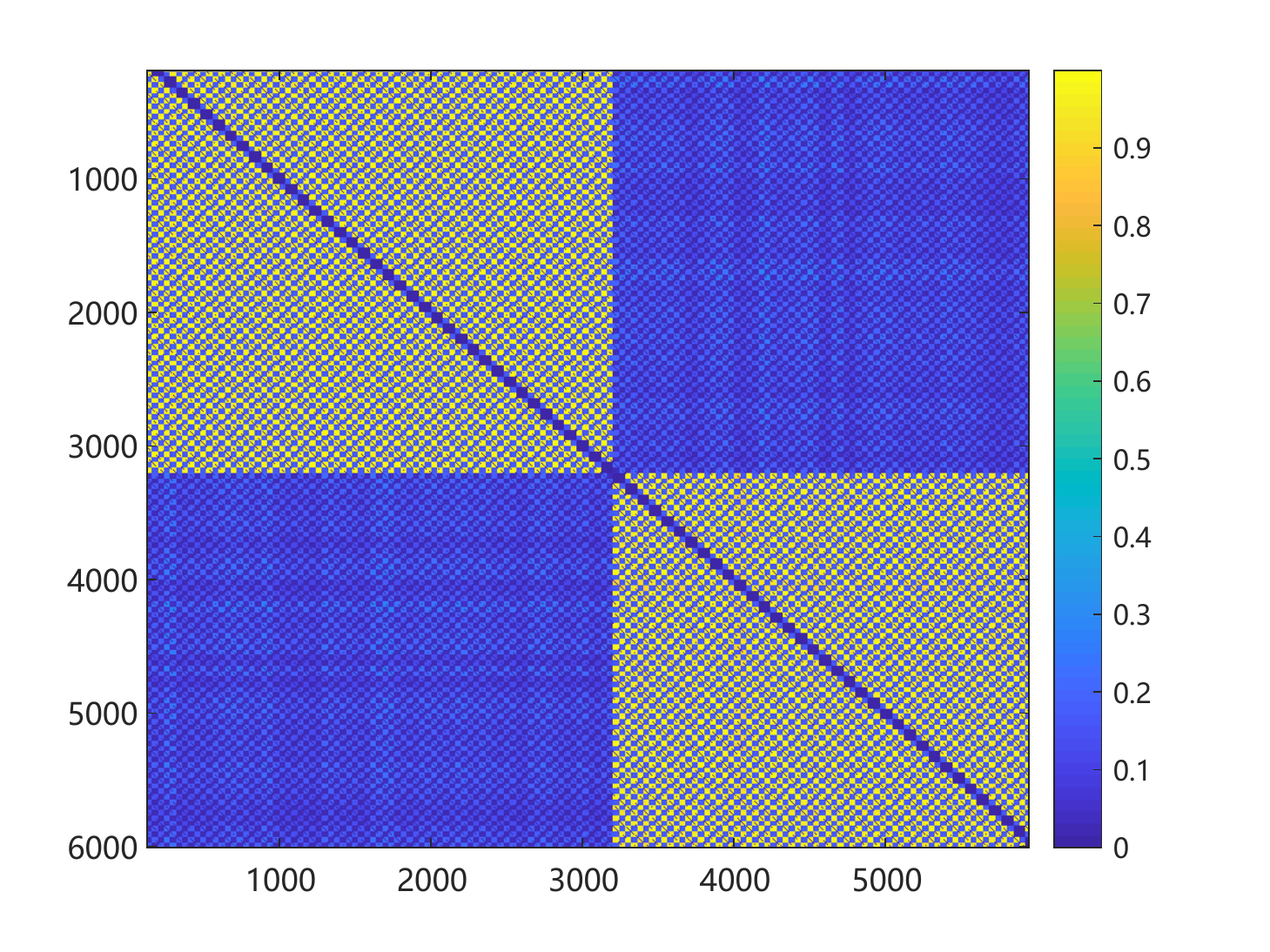}}
			\caption{ Demonstration of (a) pairwise,   (b) third-order, and  (c) fourth-order   affinity on a synthetic dataset drawn from a normal distribution. There are obvious block structures in the heatmap on different orders of unfolded affinity,   indicating that the decomposable similarities of different orders are consistent in terms of performance. }
			\label{fig3_diagG}
		\end{figure*}

		\subsubsection{Indecomposable Third-Order Tensor Affinity }
		
		After constructing the connection of decomposable affinity with pairwise similarity, we analyze the local spatial relations to obtain supplement information for pairwise relations. As ~\cite{low_embedding} stated, when the order is two, the spectral methods  can also be seen as an eigenvector $\vv$  containing the important underlying manifold information of the data.
		We generate this model to high-order affinity. Different third-order information can be extracted by examining the three samples from different views. As an example, the affinity between three samples is defined as follows: for any three samples $x_i$, $x_j$, and $x_k$, we can always use one sample $x_j$  as an anchor point for the other samples $x_i$ and $x_k$ and consider anchor point $x_j$'s perspective to measure the relationship between the two other points. If $x_i$ and $x_j$ are sufficiently similar, they are similar no matter the location of the anchor point. But for outliers, the similarity to other data should be small under  most anchors' perspectives. Therefore, the undecomposable affinity $\cT_3\in \mathbb{R}^{m^2\times m}$ among three points is defined by:

		\begin{equation}\label{eq:order-3Sim}
			\cT_{3_{ijk}}=\frac{<(x_i-x_j), (x_k-x_j)>}{d_{ij}d_{jk}}
		\end{equation}
		for $i,j,k\in m$, where $d_{ij}$ denotes the distance between samples $x_i$ and $x_j$.
		The entry $\mathcal T_{3_{ijk}}$ denotes the similarity between $x_i$ and $x_k$ under $x_j$'s view. If the distance between $x_i$ and $x_k$ is small, wherever the location of $x_j$ is, it should be a small value.
		Then, we can unfold $\mathcal T_3$ into a matrix $\hat{\bm T_3} \in \mathbb R^{m^2\times m}$ along the mode-3 direction.  Similar to the decomposable tensor affinity, we can infer the three-order normalized affinity tensor $\mathcal{L}_3 $, which when unfolded along the mode-3 direction satisfies $\hat{\bm L}_3=\hat{\bm D}_{3_1}^{-\frac{1}{2}}\hat{\bm T}_3\hat{\bm D}_{3_2}^{-\frac{1}{2}}$, where $\hat {\bm D}_{3_1}^{-\frac{1}{2}} =\sqrt{\sum_i{\hat{\bm T}_{3_{::i}}}\sum_i{\hat{\bm T}_{3_{::i}}}} =\bm D_2\otimes \bm D_2$ and $\hat {\bm D}_{3_2} = \sum_i{\hat{\bm T}_{3_{::i}}} =\bm D_2 * \bm D_2$.
		Similar to the form of the pairwise and tetradic relations, we seek a low-dimensional embedding $\vv$ from the third-order tensor affinity by maximizing the cross-entropy:
		
		\begin{equation}
			\begin{aligned} \max_{\vV} &\sum_{j=1}^k \mathcal{L}_3\otimes _3\vv_j\otimes _2\vv_j\otimes _1\vv_j\\
				=\max_{\vV} ~&  tr((\vV*\vV)^T\hat{\bm L}_3\vV)\\
				& s.t. \vV^T\vV = \bm I
			\end{aligned}
			\label{SC_Obj_L23}
		\end{equation}

		\subsection{Uniform Tensor Clustering}
		
		We now formulate a uniform framework to learn a low-dimensional embedding from  affinities of various orders.
		Suppose we have $n$ samples $\vX=[\vx_1,\vx_2,...,\vx_m] \in \mathbb{R}^{n\times m}$. The aim of clustering is to assign the samples  into $k$ disjoint clusters, $C_1,C_2,\dots,C_k$.  We introduce an entropy, called total normalized similarity, to quantify the sample assignment:
		\begin{definition}
			\textbf{The total N-similarity of a cluster:} Let $\bm C \in \bm X$ be a group of samples and $\mathcal S \in \mathbb R ^{n_1 \times n_2 \dots \times n_k}$ be the order-$k$ similarity tensor. The total normalized similarity of $\bm C$ is defined as
			\begin{equation}
				\begin{aligned}
					Sim(\bm C) = \sum_{x_{i_1},x_{i_2},\dots,x_{i_k}\in \bm C} \mathcal L_{x_{i_1},x_{i_2},\dots,x_{i_k}}
				\end{aligned}
			\end{equation}
			where $\mathcal L$ is the normalized affinity tensor calculated by $\mathcal S$. In addition, for a partition $\bm C_1, \bm C_2, \dots, \bm C_k$ of $\bm X$, we define the normalized associativity of the clustering as:
			\begin{equation}
				\begin{aligned}
					N-Assoc(\bm C_1, \bm C_2, \dots, \bm C_k) = \sum ^{k}_{i=1}\frac{Sim(\bm C_i)}{|\bm C_j|^m},
				\end{aligned}
				\label{Assoc}
			\end{equation}
			where $|\bm C_j|$ denotes the sample number of cluster $\bm C_j$.
		\end{definition}
		
		Let an indicator matrix $\mH =[\vh_1,\vh_2,\dots,\vh_k]\in \mathbb R ^{n\times k}$ be the sample assignment  such that $\bm \mH_{ij} = |\bm C_j|^{-1}$ if $x_i\in \bm C_j$ and zero otherwise.
		
		One can  maximize the normalized associativity in Eq.~(\ref{Assoc}) to seek an optimal sample assignment ${\bm C_i,\dots,\bm C_k}$. Via algebraic transformation, this problem is equivalent to solving:
		
		\begin{equation}
			\begin{aligned}
				\max_{\bm C_1,\cdots,\bm C_k} \sum_{j=1}^k(\mathcal L \otimes_m \vh_j \otimes_{m-1} \vh_j \dots \otimes_1 \vh_j)
			\end{aligned}
			\label{}
		\end{equation}
		where $\vh_j$ denotes the $j^{th}$ column of $\bm H$.
		
		One can verify that solving the maximum normalized associativity is NP-hard. Alternatively,  one can  relax the binary assignment matrix to the orthonormal matrix $\vV\in \mathbb R^{n\times k}: \vV^T\vV = \bm I$ to lower the expectation of binary assignment. The simplified  problem then becomes solving
		
		\begin{equation}
			\begin{aligned}
				\max_{\vV^T\vV = \bm I} \sum_{j=1}^k \mathcal L \otimes_1 \vv_j \otimes_2 \vv_j \dots \otimes_m \vv_j
			\end{aligned}
		\end{equation}
		where $\vv_j$ represents the $j^{th}$ column of $\vV$.
		
		We now fuse the aforementioned undecomposable third-order and fourth-order tensor affinities with the pairwise affinity matrix to seek a uniform low-dimensional embedding. The high-order tensor affinity can be used to effectively address the challenge of the vulnerability to high-dimensional datasets. The proposed model, named uniform tensor clustering (UTC),  effectively fuses the similarity measurements among  samples to yield a uniform low-dimensional embedding to approximate  the partition matrix. It achieves the goal by maximizing the cluster's associativity of various orders. The proposed model is formulated as:
		\begin{equation}
			\begin{aligned}
				\min_{\vV} \,\,& \sum_{j=1}^k -{\hat L}_2\otimes _2\vv_j\otimes _1\vv_j-\mathcal{L} _3\otimes _3\vv_j\otimes _2\vv_j\otimes _1 \vv_j \\ &-\mathcal{L} _4\otimes _4\vv_j\otimes _3\vv_j\otimes _2\vv_j\otimes _1\vv_j \\
				&s.t\,\,\vV^T\vV =\bm I
			\end{aligned}
			\label{SC_Obj_L234}
		\end{equation}
		where  $\vv_j$ is the $j^{th}$ column of the matrix $\vV$.
		This unified model aims to fuse various affinity matrices to yield a uniform low-dimensional embedding that is robust to noise contamination and high-dimensional concentration effects. The proposed model  enhances the clustering performance by exploiting spatial information from multiple (more than two) samples. Upon solving this problem, one obtains the embedding $\vV$. The clustering task is accomplished by  fusing low- to high-order information and using K-means clustering based on the embedding matrix $\vV$ to obtain the final group. Therefore, once we have the indecomposable tensor affinity $\mathcal S$ from different orders, computing the fused low-dimensional embedding yields the different orders' spatial information to enhance the clustering performance.
		
		\subsection{Numerical Scheme to Solve UTC}
		The optimization problem in Eq.~(\ref{SC_Obj_L234}) is convex, and the objective is differentiable. However, the gradient of the objective function involves solving a polynomial function of order three. It has no explicit solution and thus is computationally intractable. To circumvent this issue, we introduce a slack variable $\vv_2$ to approximate the term  $\vv\otimes\vv$, thereby avoiding having to solve a high-order polynomial function. The model can be rewritten as
		\begin{equation}
			\begin{aligned}
				\min_{\vV_1 , \vV_2}\,\, & -tr(\vV_1^T\hat {\bm L}_2\vV_1 + \vV_2^T \hat{\bm L}_3\vV_1+\vV_2^T\hat{\bm L}_4 \vV_2)\\
				s.t.\,\,& \vV_1^T\vV_1=\bm I \\
				& \vV_1 * \vV_1 = \vV_2
			\end{aligned}
			\label{ALF_form}
		\end{equation}
		
		This problem can be solved by alternating direction minimization. Its augmented Lagrangian formulation is  as follows:
		\begin{equation}
			\begin{aligned}
				\min_{\vV_1,\vV_2}\,\,& -tr(\vV_1^T\hat{\vL}_2\vV_1) - tr(\vV_2^T \hat{\bm L}_3\vV_1) - tr(\vV_2^T \hat{\bm L}_4 \vV_2) \\
				+& <\bm Y_1 ,\vV_1 * \vV_1 - \vV_2> + <\bm Y_2, \vV_1^T\vV_1 - \bm I>\\
				+& \frac{\mu}{2 } [\|\vV_1 *  \vV_1 - \vV_2\|_F^2 + \|\vV_1^T\vV_1- \bm I\|_F^2]
			\end{aligned}
			\label{ALF_form_2}
		\end{equation}
		where $\bm Y_1$ and $\bm Y_2$ are Lagrange multipliers. The constant $ \mu > 0 $ is a penalty parameter. We then decompose the problem into two subproblems  with respect to the  variables $\vV_1$ and  $\vV_2$ alternatively. Each  variable is solved while fixing other variables. The process is updated iteratively until convergence.
		
		\textbf{Step 1): Solving the subproblem with respect to the variable $\vV_1$}
		
		By fixing the variable $\vV_2$, the  problem can be  simplified as:
		\begin{equation}\label{eq:subproblem1}
			\begin{aligned}
				\min_{\vV_1} \,\, &-tr(\vV_1^T\bm L\vV_1) - tr(\vV_2^T\hat{\bm L}_3\vV_1)\\
				+ &< \bm Y_1, \vV_1 * \vV_1 - \vV_2> + <\bm Y_2, \vV^T\vV- \bm I>\\
				+ &\frac{\mu}{2}[\|\vV_1 *  \vV_1 -\vV_2\|_F^2 + (\vV^T\vV- \bm I)]
			\end{aligned}
		\end{equation}
		
		The gradient of the quadratic  function is:
		\begin{equation}\label{eq:gradient1}
			\begin{aligned}
				\frac{\partial L}{\partial \vV_1} = & - 2 \bm L \vV_1 - \hat{\bm L}_3^T\vV_2 + \mu\sum_{j=1}^k (\vV_{1_{:j}} \otimes \bm I\\
				&+\bm I\otimes \vV_{1_{:j}})^T(\vV_{1_{:j}} * \vV_{1_{:j}}-\vV_{2_{:j}}+\bm Y_{1_{:j}} /\mu)\\
				&+ 2\mu\vV_1(\vV_1^T\vV_1-\bm I+\bm Y_2/\mu)
			\end{aligned}
		\end{equation}
		
		Thus, the problem in Eq.~(\ref{eq:subproblem1}) can be solved via gradient descent or by setting the gradient function Eq.~(\ref{eq:gradient1}) to be zero to obtain the solution directly.

		\textbf{Step 2): Solving the subproblem with respect to the variable $\vV_2$}
		
		By fixing the variable  $\vV_1$, the augmented Lagrange function can be simplified as:
		\begin{equation}
			\begin{aligned}
				\min_{\vV_2} \,\, &tr(-\vV_2^T\hat{\bm L}_3 \vV_1 ) - tr(\vV_2^T\hat{\bm L}_4\vV_2 )\\
				&+\bm Y_1(\vV_1 * \vV_1 - \vV_2)+\frac{\mu}{2}[\|\vV_1 * \vV_1 -\vV_2\|_F^2
			\end{aligned}
		\end{equation}
		The gradient of the objective function is:
		\begin{equation}\label{eq:gradient2}
			\begin{aligned}
				\frac{\partial L}{\partial \vV_2} = -\hat{\bm L}_3 \vV_1-2\hat{\bm L}_4\vV_2 +\mu(\vV_2-\vV_1 * \vV_1-\frac{\bm Y_1}{\mu})
			\end{aligned}
		\end{equation}
		By setting the gradient to zero, one can obtain the implicit solution as:
		\begin{equation}\label{eq:ClosedFromV2}
			\begin{aligned}
				\vV_2^{*} = (\mu\bm I-2\hat{\bm L}_4)^{-1}(\mu \vV_1 * \vV_1 + \hat{\bm L}_3\vV_1 + \bm Y)
			\end{aligned}
		\end{equation}
		
		\textbf{Step 3): Updating the multipliers $\bm Y_1$ and $\bm Y_2$:}
		
		\begin{equation}\label{eq:updateY1}
			\begin{aligned}
				\bm Y_1^{(k+1)} = \bm Y_1^{(k)} + \mu (\vV_1^{(k)} * \vV_1^{(k)} -\vV_2^{(k)})
			\end{aligned}
		\end{equation}
		\begin{equation}\label{eq:uqdataY2}
			\begin{aligned}
				\bm Y_2^{k+1} = \bm Y_2^{(k)} +\mu ({\vV_1^{(k)}}^T\vV^{(k)}-\bm I)
			\end{aligned}
		\end{equation}
		The three steps are  iteratively updated until convergence or until a stopping criterion is met: $max(\|\vv_1^{k+1}-\vv_1^k\|_\infty, \|\vv_2^{k+1}-\vv_2^k\|_\infty,\|\vv_1^{k+1}\otimes \vv_1^{k+1} - \vv_2^{k+1}\|_\infty)<\epsilon$. Since the main problem is convex, the numerical scheme is guaranteed to be convergent.
		The  alternate strategy to solve the problem is summarized in Algorithm \ref{ALg_1}.
		
		\begin{algorithm}[t]
			\caption{High-Order Affinity Clustering}
			\hspace*{0.02in} {\bf Input:}
			Data $\mathcal X\in \mathbb R^{m*d}$, Cluster number $c$\\
			\hspace*{0.02in} {\bf Output:}
			The fused low-dimensional embedding $\vV$ and the cluster results $\bm Y_{pred}$
			\begin{algorithmic}[1]
				\State Compute the pairwise affinity $\hat{\vL}_2$. Set $\vV_1^0 = \vV_2^0 = \bm Y_1^0 = \bm Y_2^0 = \bm 0$, $\mu^0 = 10^{-3}$, $\mu^{max} = 10^2$, $\epsilon = 10^{-2}$, $\varepsilon = 10^{-3}$, $\rho = 1.1$, $\alpha = 10^{-3}$, $k=t=0$, and $\sigma = 10^{-6}$.
				\State Construct triadic affinity tensor $\mathcal L_3$ by Eq.~(\ref{eq:order-3Sim}).
				\State Construct tetradic affinity tensor $\mathcal L_4$ by Eq.~(\ref{eq:order-4Sim}).
				
				\State Unfold order-3 tensor $\mathcal L_3$ to matrix $\hat{\bm L_3} $ by Eq.~(\ref{unfolding_order_3}).
				\State Unfold order-4 tensor $\mathcal L_4$ to matrix $\hat{\bm L_4} $ by Eq.~(\ref{unfolding_order_4}).
				\While{Not Converged}
				\State $\vV_{1_{temp}}^t = \vV_1^k$.
				
				\While{Not Converged}
				\State Fix $\vV_2^k$, computed $\frac{\partial L}{\partial\vV_{1_{temp}}^t}$ by Eq.~(\ref{eq:gradient1}).
				\State Update $\vV_{1_{temp}}^{t+1}$ by $\vV_{1_{temp}}^{t+1} = \vV_{1_{temp}}^{t} + \alpha \frac{\partial L}{\partial\vV_{1_{temp}}^{t}}$.
				\State Check the convergence conditions:
				
				~~~$\|\vV_{1_{temp}}^{t+1} - \vV_{1_{temp}}^{t}\|_\infty \le \epsilon $ .
				\State $t= t+1$.
				\EndWhile
				\State $\vV_1^{k+1}  =\vV_{1_{temp}}^{t} $
				\State Fix $\vV_1^{k+1}$, update $\vV_2^{k+1}$ by Eq.~(\ref{eq:ClosedFromV2}).
				\State Update $\bm Y_1^{k+1}$ by Eq.~(\ref{eq:updateY1}).
				\State Update $\bm Y_2^{k+1}$ by Eq.~(\ref{eq:uqdataY2}).
				\State $\mu^{k+1} = min(\rho\mu^{k}, \mu^{max})$
				\State Check the convergence conditions:
				
				$\|\vV_1^{k+1} - \vV_1^k\|_\infty \le \varepsilon$, $\|\vV_2^{k+1} - \vV_2^k\|_\infty \le \varepsilon$
				
				$\|\bm Y_1^{k+1} - \bm Y_1^k\|_\infty \le \varepsilon$,$\|\bm Y_2^{k+1} - \bm Y_2^k\|_\infty \le \varepsilon$
				\State $k = k+1$.
				\State $t = 0$
				\EndWhile
				
				\State \Return $\vV_1^k$
				\State Perform spectral clustering on $\vV\vV^T$ to obtain $\bm Y_{pred}$
			\end{algorithmic}
			\label{ALg_1}
		\end{algorithm}
		
		\section{EXPERIMENTS}
		In this section, we illustrate the superiority of the proposed UTC method on both synthetic and real datasets. We first construct an example (Syndata-1) to demonstrate the scenario in which clustering through the standard pairwise affinity fails. In comparison, the triadic affinity can be employed to adequately characterize the  spatial distribution of multiple samples; thus, the proposed UTC can perfectly cluster the sample using the triadic affinity. We then construct a second synthetic dataset (Syndata-2) with the HDLSS property to demonstrate the scenario in which  clustering through standard pairwise affinity is unsatisfactory when the feature dimension increases. In comparison,  the proposed UTC can perform clustering tasks stably by integrating affinities of different orders.   Finally, we conduct experiments on six real datasets to verify the effectiveness of the proposed UTC by comparison with five popular clustering methods.
		
		\subsection{Experimental Settings}
		We employ five popular clustering methods to benchmark the proposed UTC. A brief introduction to these methods is given below:
		\begin{itemize}
			\item{\textbf{Spectral clustering (SC)}~\cite{Ncut}}: The classic spectral clustering gives a baseline on behalf of pairwise similarity.
			
			\item{\textbf{Clustering with Adaptive Neighbors (CAN)} ~\cite{CAN}}: Dynamically learning the affinity matrix by assigning the adaptive neighbors for each data point based on local distance. The final similarity matrix's connected components are made precisely equal to the cluster number by imposing a new rank constraint.
			
			\item{\textbf{Clique averaging+ncut (CAVE)} ~\cite{CAGE}}: The hypergraph, on behalf of the high-order affinity, is approximated using clique averaging, and the resulting graph is partitioned using the normalized cuts algorithm.
			
			\item{\textbf{Pair-to-pair clustering (PPC)}~\cite{ref15}}: Tetradic undecomposable affinities are used to obtain the low-dimension embedding as the final similarity matrix of the spectral method.
			
			\item{\textbf{Integrating tensor similarity and pairwise similarity (IPS2)}~\cite{ref15}}: Applying the spectral clustering method on the fused similarity that combines the pair-to-pair affinities and pairwise similarity.
			
			\item{\textbf{Uniform Tensor Clustering (UTC)}}: The proposed method.
		\end{itemize}
		In terms of the type of the affinity, the five methods can be divided into three categories, i.e., pairwise affinity, high-order affinity, and fused-order affinity.
		
		The performance of each method is evaluated with respect to five popular metrics:  accuracy (ACC), adjusted Rand index (ARI), F-score, normalized mutual information (NMI), and purity (PURITY).   The larger the value is, the better the performance is.

		\begin{figure*}[]
			\centering
			\subfigure[]{
				\includegraphics[width=5cm]{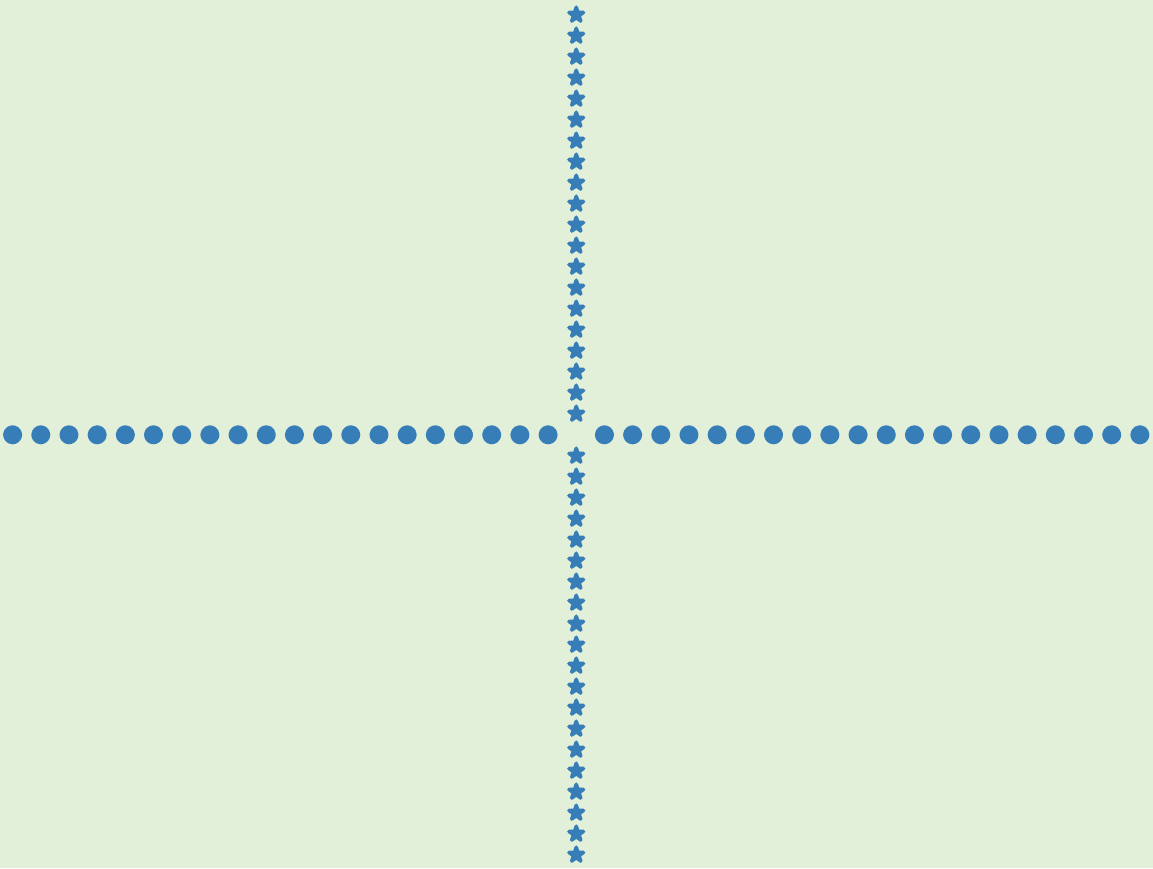}
			}
			\subfigure[]{
				\includegraphics[width=5cm]{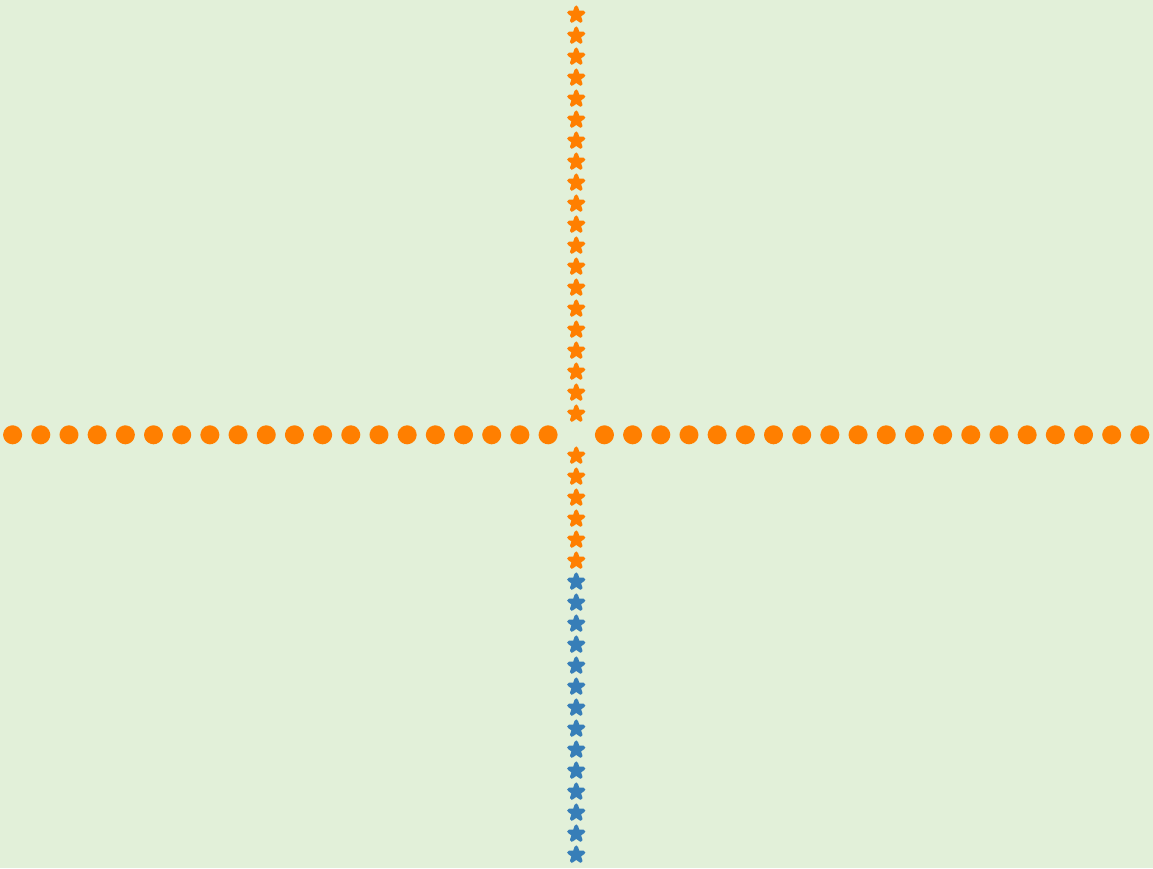}
			}
			\subfigure[]{
				\includegraphics[width=5cm]{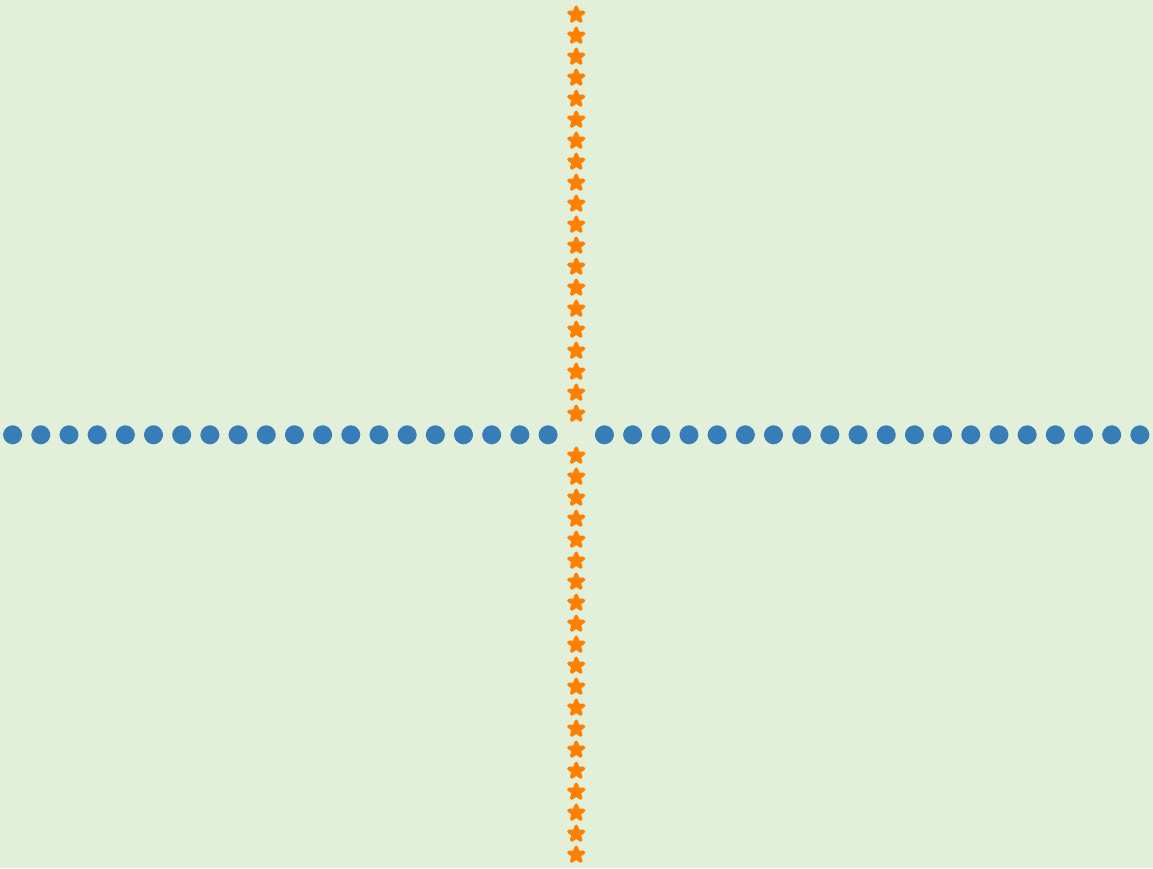}
			}
			\caption{Experiment on Syndata-1. (a)  Syndata-1 is constructed to consist of two subgroups, which are perfectly distributed along two cross-lines, highlighted by $\cdot$ and $*$; (b) The assignments obtained using pairwise affinity mistakenly mix the two groups, highlighted in orange   and blue. Most of the samples are labeled incorrectly; (c) The assignments obtained by UTC using  triadic affinity perfectly segment  the samples into the correct subgroups, highlighted in orange  and blue. }
			\label{fig2_syndata1}
		\end{figure*}
		
		\subsection{Experiment on Synthdata-1 to Validate the Discrimination Potential of UTC via Triadic Affinity}
		\textbf{Syndata-1} consists of 40 samples from two different subgroups. Each group is perfectly distributed along a straight line, while the two groups are distributed perpendicularly to each other. We use circles and crosses to denote the samples from the two subgroups, and their distribution is shown in Fig. ~\ref{fig2_syndata1}(a).
		
		We applied the  spectral clustering, which works on pairwise affinity, on  \textbf{Syndata-1}. Since the samples from the two subgroups are perpendicular, the resulting affinity matrix is non-differentiable. Therefore, the SC fails to yield the correct assignment. It mistakenly classifies the samples into two parts, with most in the upper and the rest in the bottom. In comparison, UTC can identify the sample assignments perfectly with $100\%$ accuracy. The results indicate that the proposed undecomposable high-order affinity can explore complementary information to improve the clustering performance.

		\begin{figure*}[]
			\centering
			\includegraphics[width=17cm]{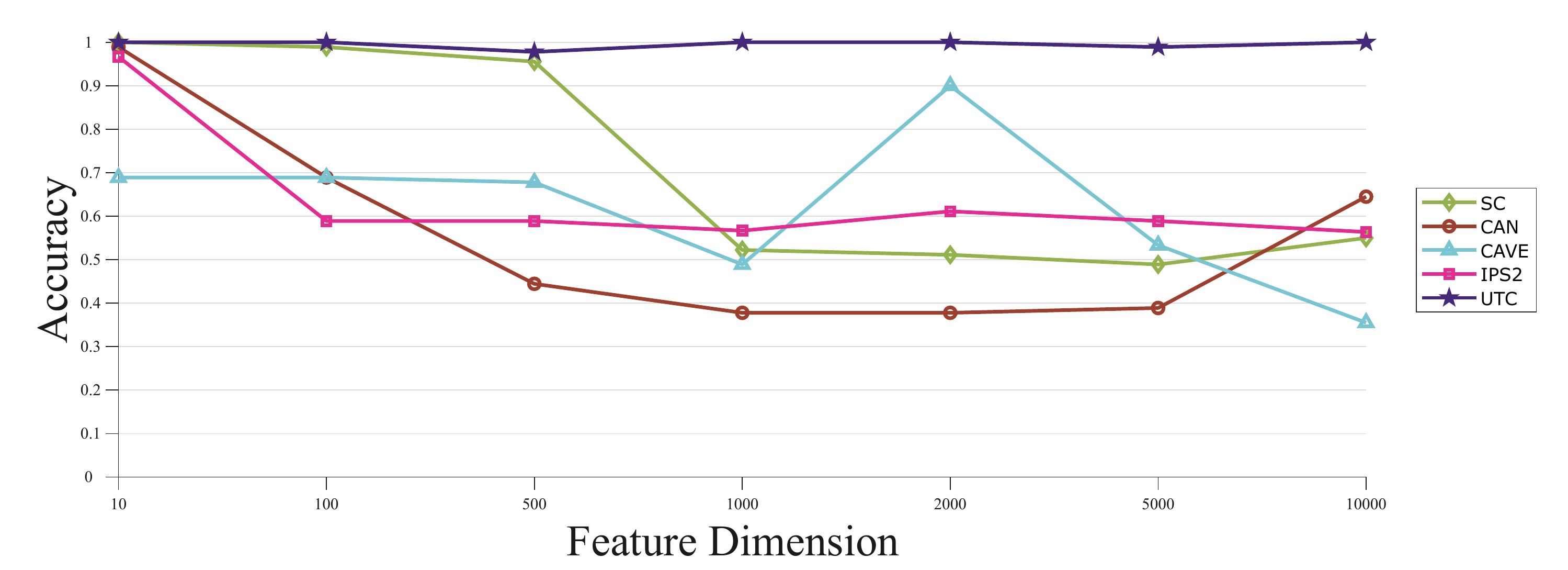}
			\caption{Experiment to demonstrate the stability of UTC when dealing with high-dimensional feature corruption. The feature dimension is artificially increased from 10 to 10000. The accuracy of most of the comparative methods gradually decreased from 100\% to around 35\%. In comparison, the performance of UTC was stable; the accuracy was maintained at approximately 100\% as the dimension increased.}
			\label{syn2_acc_bar}
		\end{figure*}

		\begin{table*}[]
			\caption{Clustering results on Syndata-2 when the dimensionality is 10000}\label{tab:syndata2}
			\centering
			\renewcommand\tabcolsep{10.0pt}
			\begin{tabular}{cclllll}
				\hline
				\multicolumn{1}{l}{Datasets} & Methods                & \multicolumn{1}{l}{ACC}&\multicolumn{1}{l}{ARI}&\multicolumn{1}{l}{F-SCORE}&\multicolumn{1}{l}{NMI}&\multicolumn{1}{l}{PURITY}\\ \hline \hline
				\multirow{8}{*}{Syndata2}  & SC                  & 0.55   & 0.4427 & 0.6633 & 0.6544 & 0.68  \\
				& CAN                    & 0.6449 & 0.2306 & 0.5342 & 0.3323 & 0.6449  \\
				& CAVE                   & 0.3551 & -0.013 & 0.3734 & 0.0018 & 0.3551\\
				& PPS                    & 0.5636  & 0.3166  & 0.4472 & 0.4453 & 0.5636\\
				& IPS2                   & 0.7218  & 0.4847  & 0.5848 & 0.5862 & 0.7818\\
				& UTC(Fusing pairwise and triadic affinity)       & 0.86    & 0.6719  & 0.7865 & 0.759  & 0.86\\
				& UTC(Fusing pairwise and tetradic affinity)      & 0.77    & 0.5204  & 0.6799 & 0.5776 & 0.77\\
				& UTC    & $\bm{1}$ &$\bm{1}$ &$\bm{1}$ &$\bm{1}$ &$\bm{1}$\\ \hline
				\hline
			\end{tabular}
		\end{table*}

		\begin{figure*}[]
			\centering
			\includegraphics[width=17cm]{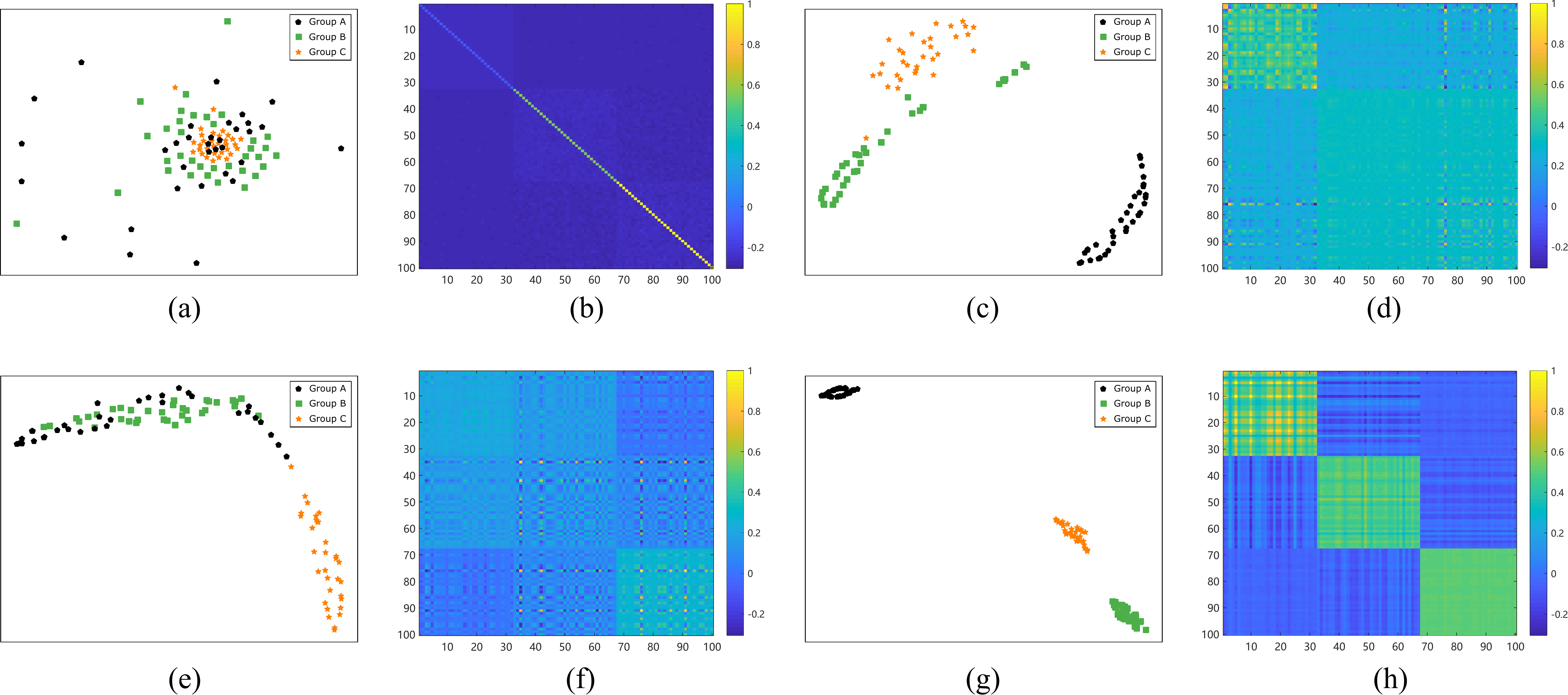}
			\caption{Visualization results on Syndata2 with a feature dimensionality of 10000. We apply t-SNE on the embedding features derived from UTC and the corresponding affinity matrix in three ways. Clearly, most of the samples are mixed together when applying t-SNE to the raw data(a), and the corresponding affinity matrix has no clear-cut structure (b). The samples in group A are well-segmented, while the other groups remain undistinguished when fusing the triadic affinity with pairwise affinity (c). The upper left of the corresponding affinity matrix has blocky structures (d). Although group C can be identified accurately by combining fourth-order association with pairwise similarity, the remaining two types of the Gaussian mixture distribution are merged (e),(f). By combining the third-order and fourth-order affinities with pairwise similarity, each group can be identified accurately (h), and three distinct blocks are presented on the diagonal of the corresponding affinity matrix (g).}
			\label{vis_syn2}
		\end{figure*}

		\subsection{Experiment  on Synthdata-2 to Validate the Stability of UTC over Dimension Expanding via High-order Affinities}
		
		\textbf{Syndata-2} consists of three subgroups, each derived from independent and identically distributed normal distributions with an equal standard deviation of 0.5 and mean of 2. The data from different groups are orthogonal, and each group contains 30-40 samples.

		The pairwise affinity tends to be ambiguous with expanding feature dimensionality, limiting the clustering performance. We applied the proposed UTC on \textbf{Syndata-2} to validate its stability in handling high-dimensional data by comparison with popular methods, including SC, CAN and CAVE, PPS, and IPS2. The results are shown in Fig.~\ref{syn2_acc_bar}. When the feature dimension varies from 10 to 10000, UTC can maintain robustness to perform accurate clustering. Furthermore, the baseline methods gradually lose the ability to identify the data structure.
		
		Moreover, when the number of data dimensions reaches 10000, Tab.~\ref{tab:syndata2} shows that our proposed UTC methods, including fusing pairwise with triadic affinity, tetradic affinity, and both,  substantially outperform all other state of the art methods. The ACC clustering results of fusing triadic and fusing tetradic are higher than those of the second-highest baseline method by 14\% and 28\%, demonstrating the effectiveness of fusing high-order affinity with pairwise relation under HDLSS data. Additionally, the ACC achieves 100\% when utilizing both higher-order affinities, proving that fusing different-order affinities is more helpful in completely explaining the data structure than is merging only a single higher-order affinity.
		
		We also employ the t-SNE to visualize the differences in the embedded features after different methods. Most of the samples mix together when visualizing their embedding by t-SNE on the raw data, as  shown in Fig.~\ref{vis_syn2}(a). On the other hand, the pairwise affinity, shown in Fig.~\ref{vis_syn2}(b), has no clear-cut structure, except the diagonal elements having large values.   Therefore, it is challenging to discriminate the three subgroups when using the pairwise affinity matrix directly.
		Now, we fuse the three-order tensor affinity with pairwise affinity and apply UTC  to obtain a local embedding. t-SNE is applied to the resulting embedding to visualize the sample assignment distribution. The visualized result is shown in Fig.~\ref{vis_syn2}(c). The samples in one subgroup (Subgroup A) are clearly separated from the others,  while the remaining two subgroups (Subgroups B and C) are not well-distinguished. The upper left of the affinity matrix shows an obvious blocky structure, but the remaining part lacks a clear structure, as shown in Fig.~\ref{vis_syn2}(d), which confirms the results obtained by t-SNE. Alternatively, if we fuse the fourth-order tensor affinity with the pairwise matrix affinity and apply UTC to obtain another local embedding,  the visualization using t-SNE is shown in Fig.~\ref{vis_syn2}(e). The samples in Subgroup C clearly are separated from the other two subgroups, while the other two subgroups are merged. The affinity heatmap also validates the observation that Subgroup A in the lower-right corner possesses distinct value compared to the others, as shown in Fig.~\ref{vis_syn2}(f). Finally, we fuse the pairwise affinity with the third-order and fourth-order. UTC is employed to learn a uniform embedding that is then visualized via t-SNE, as shown in Fig.~\ref{vis_syn2}(g). The three subgroups are perfectly separated in this case. The diagonal elements of the corresponding affinity matrix also show three obvious block-like structures and an affinity difference between the three subgroups. Therefore, an accurate sample assignment can easily be obtained from the embedding by the proposed UTC. In summary, the above experimental results validate that the high-order tensor affinities make up for the insufficiency of the pairwise affinity matrix when corrupted by high-dimensional features.

		\begin{table}[]
			\caption{Statistics on six tested real-world datasets.}
			\begin{center}
				\begin{tabular}{c|c|c|c|c}
					\hline
					Dataset   &   Type           &Samples   & Features  &   Clusters     \\\hline
					Yale      &   Facial Image          &55        & 4096      &   5            \\
					Coil20    &   Object Image        &100       & 16384     &   5            \\
					Lymphoma  &   Bioinformatics &62        & 4702      &   3            \\
					DriveFace &   Facial Image   &78        & 307200    &   3            \\
					Lung      &   Bioinformatics &100       & 12600     &            3            \\
					GLI-85	  &   Bioinformatics &85        & 22283     &   2            \\\hline
				\end{tabular}
			\end{center}
			\label{tab:dataset}
		\end{table}

		\subsection{Experiments on Real Datasets}
		
		We then validated the power of UTC by applying it to six public benchmark datasets. The six datasets are chosen to be representative of various sources, including facial image, object image, and bioinformatics data. The statistics for the six datasets are summarized in Tab.~\ref{tab:dataset}. Additionally, a detailed description is provided.

		\begin{table*}[]
			\caption{Clustering performance on six real datasets}\label{tab:results}
			\centering
			\renewcommand\tabcolsep{10.0pt}
			\begin{tabular}{cclllll}
				\hline
				\multicolumn{1}{l}{Datasets} & Methods                & \multicolumn{1}{l}{ACC}&\multicolumn{1}{l}{ARI}&\multicolumn{1}{l}{F-SCORE}&\multicolumn{1}{l}{NMI}&\multicolumn{1}{l}{PURITY}\\ \hline \hline
				\multirow{8}{*}{YALE}  & SC                     & 0.4727 & 0.2528 & 0.423  & 0.4097 & 0.4909  \\
				& CAN                    & 0.5091 & 0.3316 & 0.4839 & 0.5046 & 0.5273  \\
				& CAVE                   & 0.5818  & 0.2871  & 0.4337 & 0.4307 & 0.6182\\
				& PPS                    & 0.5636  & 0.3166  & 0.4472 & 0.4453 & 0.5636\\
				& IPS2                   & 0.7218  & 0.4847  & 0.5848 & 0.5862 & 0.7818\\
				& UTC    & $\bm{0.8}$ &$\bm{0.5019}$ &$\bm{0.6036}$ &$\bm{0.67}$ &$\bm{0.8}$\\
				\hline
				\multirow{8}{*}{Coil}  & SC                     & 0.83 & 0.9068 & 0.9248 & 0.9352 &0.96 \\
				& CAN                    & 0.86 & 0.7660 & 0.7807 & 0.8886 &0.81 \\
				& CAVE                   & 0.84 & 0.6852 & 0.7462 & 0.7575 &0.84 \\
				& PPS                    & 0.93 & 0.8383 & 0.8694 & 0.869  &0.93 \\
				& IPS2   & $\bm{0.97}$  &$\bm{0.9280}$ & $\bm{0.9419}$ & $\bm{0.9460}$ & $\bm{0.97}$\\
				& UTC    & $\bm{0.97}$  &$\bm{0.9280}$ & $\bm{0.9419}$ & $\bm{0.9460}$ & $\bm{0.97}$\\
				\hline
				\multirow{8}{*}{DriveFace}  & SC                     & 0.641  & 0.3304 & 0.5921 & 0.4505 & 0.641  \\
				& CAN                    & 0.7436  & 0.3671& 0.5941 & 0.4486 & 0.7436  \\
				& CAVE                   & 0.8974  & 0.7274& 0.8182 & 0.7449 & 0.8974\\
				& PPS                    & 0.7436  & 0.4655& 0.6433 & 0.5144 & 0.718 \\
				& IPS2                   &  0.7436 & 0.4711& 0.6549 & 0.5731 & 0.7436 \\
				& UTC   & $\bm{0.9615}$ &$\bm{0.8869}$ &$\bm{0.9237}$ &$\bm{0.8675}$ &$\bm{0.9615}$\\
				\hline
				\multirow{8}{*}{Lymphoma} & SC                     & 0.8065 &0.5059 &0.7167 &0.6634 &0.8548 \\
				& CAN                    & 0.9839 &0.9471 &0.9733 &0.9255& 0.9839 \\
				& CAVE                   & 0.8065 &0.5060 &0.7167&0.6634& 0.8548  \\
				& PPS                    & 0.9194 &0.7570 &0.8696&0.7823& 0.9194  \\
				& IPS2                   & 0.9839 &0.9471 &0.9733&0.9255& 0.9839  \\
				& UTC    & $\bm{1}$ &$\bm{1}$ &$\bm{1}$ &$\bm{1}$ &$\bm{1}$ \\
				\hline
				\multirow{8}{*}{Lung}  & SC                     & 0.8  & 0.5855 & 0.675  & 0.6848 & 0.81 \\
				& CAN                    & 0.77 & 0.5973 & 0.6971 & 0.6995 & 0.85 \\
				& CAVE                   & 0.8  & 0.5855 & 0.675  & 0.6848 & 0.81  \\
				& PPS                    & 0.73 & 0.541  & 0.6472 & 0.6311 & 0.8  \\
				& IPS2                   & 0.82 & 0.541  & 0.7427 & $\bm{0.7289}$ & 0.85  \\
				& UTC    & $\bm{0.85}$ & $\bm{0.6941}$ & $\bm{0.7634}$ & 0.7219 & $\bm{0.85}$ \\
				\hline
				\multirow{8}{*}{GLI-85}  & SC                   & 0.6471 & 0.0731  & 0.5721 & 0.1385 & 0.6941 \\
				& CAN                    & 0.6588 & -0.0353 & 0.697  & 0.042  & 0.6941 \\
				& CAVE                   & 0.7412 & 0.1225 & 0.7376 & 0.171 & 0.7412 \\
				& PPS                    & 0.5294 & -0.0447 & 0.5777 & 0.1103 & 0.6941  \\
				& IPS2                   & 0.7294 & 0.2012  & 0.6267 & 0.2754 & 0.7294  \\
				& UTC    & $\bm{0.8118}$ & $\bm{0.3772}$ & $\bm{0.7177}$ & $\bm{0.2798}$ & $\bm{0.8118}$ \\
				\hline
				\hline
				
			\end{tabular}
		\end{table*}

		\begin{itemize}
			
			\item{\textbf{Yale}}  dataset has 165 grayscale images covering 15 different individuals. Each individual has 11 different facial and configuration images, such as with glasses, without glasses, center-light, left-light, happy, sad, and so on. We extracted 4096-dimensional raw pixel values of every image in 5 classes for our experiment.
			
			\item{\textbf{Coil20}} dataset contains 1440 images of 20 categories of objects. Each category has 72 images from different views. We sample 20 samples from 5 classes, and all images have a size of $128\times 128$.

			\item{\textbf{Lymphoma}} dataset, one of the most common subtypes of non-Hodgkin's lymphoma has two molecularly different forms of diffuse large B-cell lymphoma (DLBCL), which have gene expression patterns that indicate different stages of B cell differentiation. The dataset contains a total of 62 samples, 4702 based on expression fragments.

			\item{\textbf{DriveFace}} dataset contains images sequences of subjects while driving in real scenarios. It is composed of 606 samples of $6400$ features each, acquired over different days from 3 drivers with various facial features, such as glasses and beards. For each type, we pick 26 samples for the experiment.
			
			
			\item{\textbf{Lung}} dataset contains a total of 203 samples that can be divided into 5 classes, with 149, 21, 20, 6, and 17 samples, respectively. Each sample has 12,600 genes. We selected 40 samples of the first category and all samples of the remaining four categories for experiments.
			
			\item{\textbf{GLI-85}} dataset consists of the transcriptional data of 85 diffuse infiltrating gliomas from 74 patients. Those gliomas can be divided into two kinds of tumor subclasses. Each instance has 22283 features.
		\end{itemize}

		\begin{figure*}
			\centering
			\includegraphics[width = 17cm]{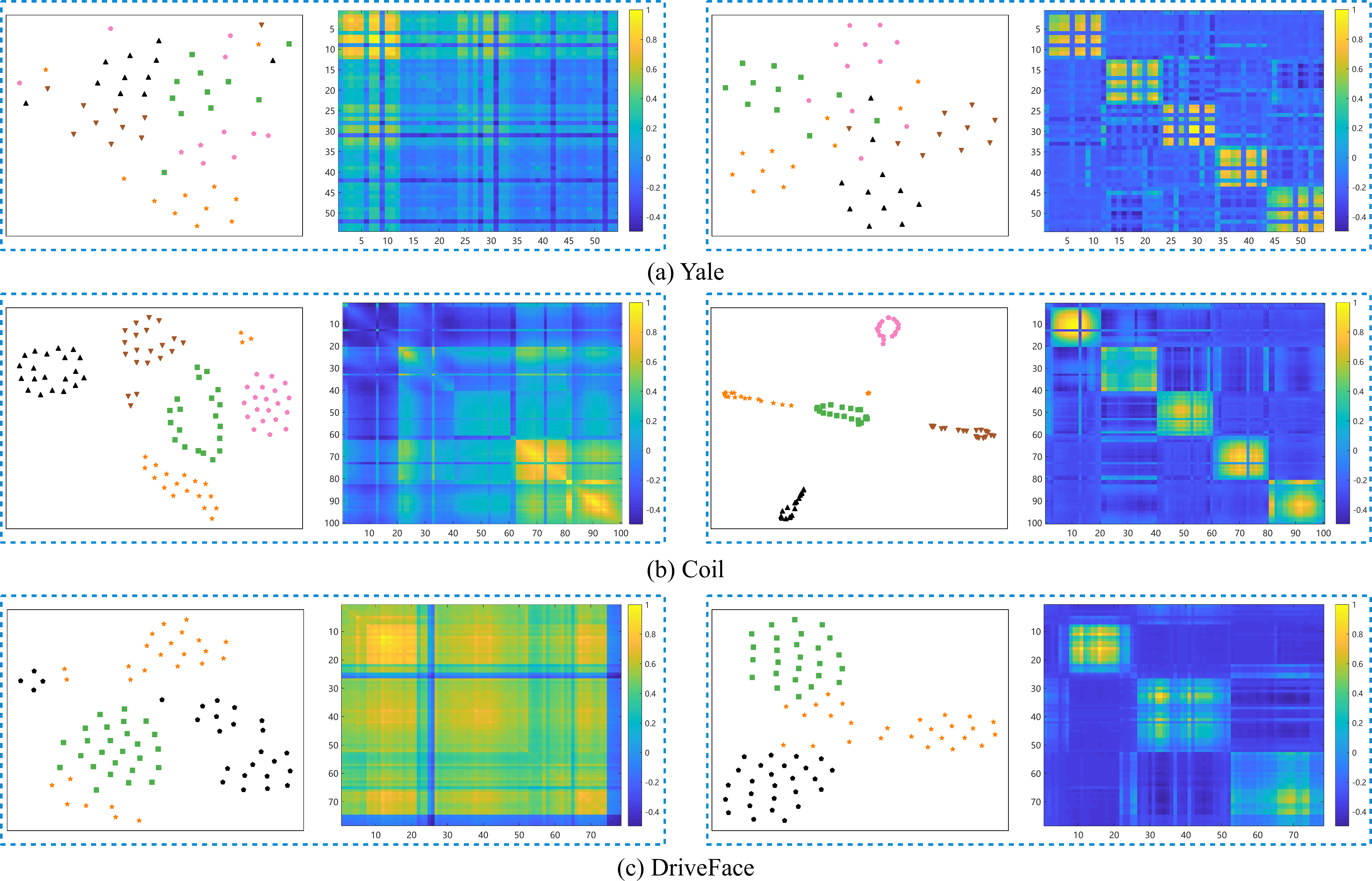}
			\caption{Visualization of the latent representations with t-SNE and corresponding affinity heatmaps of the (a) Yale; (b) Coil; and (c) DriveFace datasets. The first and second columns are the visualizations of the embedding by t-SNE on the raw data and the samples' affinity heatmap. The third and the fourth columns are the visualization results   of the obtained uniform embedding by UTC and the samples' affinity heatmap, respectively.}
			\label{Fig:tsne_1}
		\end{figure*}
		
		\begin{figure*}
			\centering
			\includegraphics[width = 17cm]{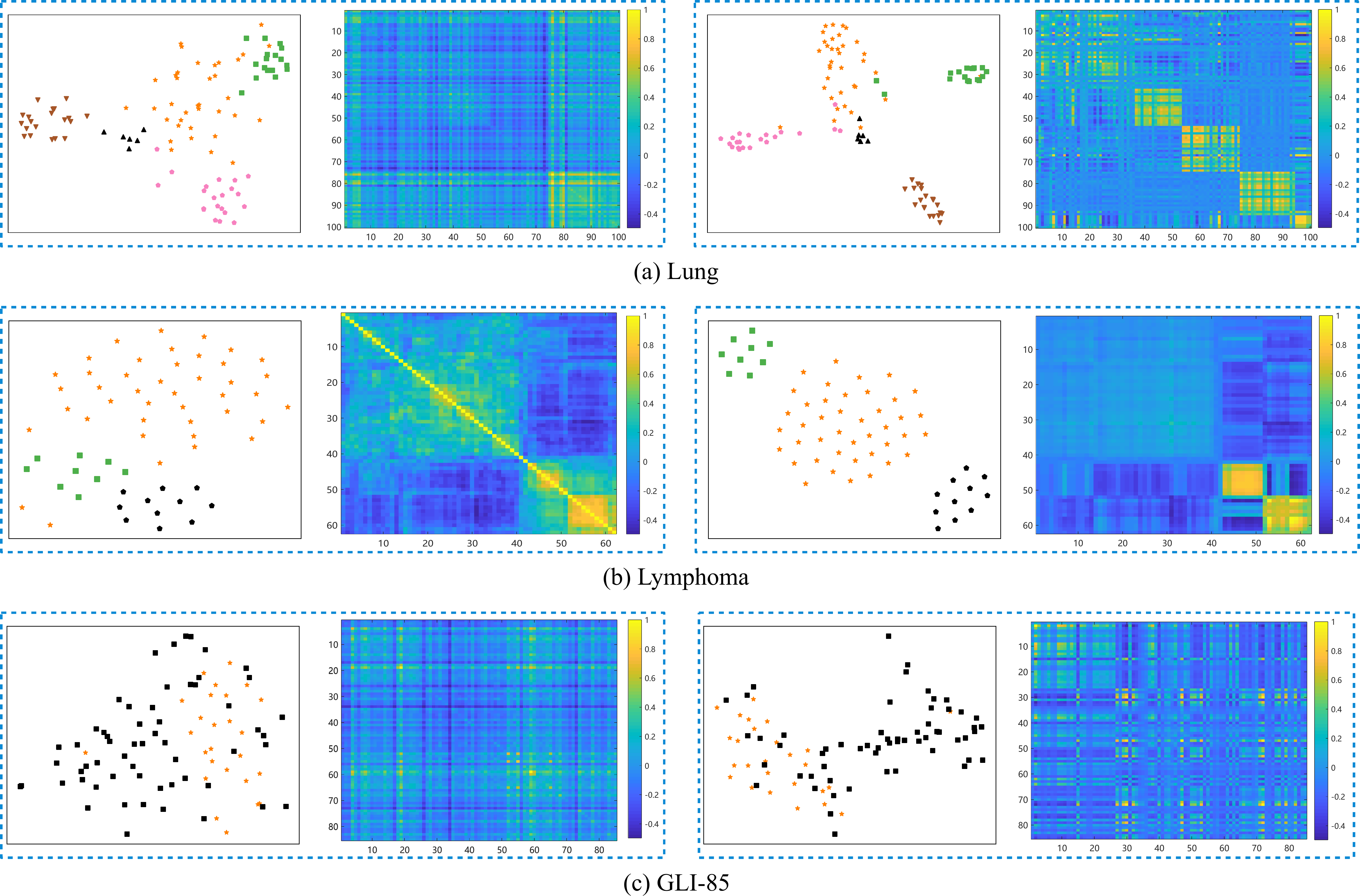}
			\caption{Visualization of the latent representations with t-SNE and corresponding affinity heatmaps of the (a) Lung; (b) Lymphoma, and (c) GLI-85 datasets. The first and second columns are the visualization   by t-SNE of the raw data and the samples' affinity heatmap.  The third and fourth columns are the visualization results   of the obtained uniform embedding by UTC and the samples' affinity heatmap, respectively.}
			\label{Fig:tsne_2}
		\end{figure*}

		We applied the proposed UTC as well as five comparative methods on the six datasets, and the results are summarized in Tab.~\ref{tab:results}. The bold value represents the best result. The clustering method that uses higher-order affinity uniformly outperforms the standard method using pairwise affinity. Therefore, the higher-order affinity constructed from the relationships of multiple samples can better describe the spatial structure of the data, which is advantageous when dealing with HDLSS data. Furthermore, our proposed UTC method significantly outperforms the competitive methods with respect to all five evaluation metrics for all types of data, including facial image, object image, and bioinformatics data.
		
		On the Yale dataset, the ACC clustering result of our method is over 33\% higher than that of SC and over 8\% higher than the second-highest IPS2. On the GLI-85 dataset, our results are more than 16\% and 6\% higher than those of the baseline and second-highest IPS2 methods. There are also 1\%, 7\%, 10\%, 2\%, and 3\% improvements compared with the second-best performance on the Coil, DriveFace, Colon, Lymphoma, and Lung datasets, respectively. IPS2 and UTC achieve the top two results on almost all datasets with respect to nearly all the metrics.
		
		We also visualized the sample spatial distribution after t-SNE\cite{t-SNE} and the heatmap of the affinity matrix on the  Yale, Coil, and DriveFace datasets in Fig.~\ref{Fig:tsne_1} and the Lymphoma, Lung, and GLI-85 datasets in Fig.~\ref{Fig:tsne_2}, respectively. The spatial distributions after t-SNE on the raw data and the resulting embedding are shown on the left side of the two columns, while the heatmaps of the raw data and the affinity matrix fusing with various orders are on the other side. In most figures, the low-dimensional embedding after UTC displays better separation and less overlap, thus making it easy for clustering.
		In comparison, the heatmap of similarity produced by SC and UTC further demonstrates the effectiveness of our method. Deep yellow represents a large proximity value. The affinity heatmap on the raw samples possesses highly blurred  boundaries, and there is no apparent block structure. However, the affinity matrix calculated from the low-dimensional embedding after UTC has clear-cut boundaries. The large value in the matrix generated by UTC is concentrated on the diagonal blocks, implying that the fused affinity can more accurately reflect the real relationship among the samples, thus yielding superior clustering performance in Tab.~\ref{tab:results}.
		Taking the Yale dataset as an example, the ACC of the clustering results obtained by UTC is 33\% higher than that of SC. The visualization results confirm our method's superiority intuitively, as shown in Fig.~\ref{Fig:tsne_1}(a). Most of the samples are mixed together in the visualization based on pairwise affinity, resulting in an inability to distinguish between different subgroups accurately, as shown on the far left of Fig.~\ref{Fig:tsne_1}(a). The heatmap of pairwise affinity confirms that only one cluster is significantly different from the other samples. Correspondingly, in the t-SNE visualization of the resulting embedding, most samples are clustered with their same cluster's partners. Five obvious blocks are displayed on the diagonal in the heatmap. The right side of Fig.~\ref{Fig:tsne_1}(a) confirms that UTC can better reflect the real structure of the data compared with single-order methods, which is critical for clustering.
		
		In summary,  the reason for the superior performance of UTC is that it utilizes the complementarity of high- and low-order affinities to comprehensively capture the spatial structure of data, thereby leading to a decisive effect on the clustering performance.
		
		\subsection{Computational Complexity Analysis}
		The UTC consists of two major computational components: constructing the high-order tensor affinity and solving the minimization problem in Eq.~(\ref{ALF_form_2}). In the first component, we calculate both the triadic and tetradic affinities by means of their $k$-nearest neighbors (KNN). Given $m$ samples, one can use the KNN strategy to reduce the time complexity from $\mathcal{O}(m^4)$ to $\mathcal O(mk^4)$. Thus, the two affinity tensors are sparse, with the number of nonzero elements being $s$. For the latter, we solve th model by means of an iterative strategy. The subproblem of $\vV_1$ based on the gradient descent strategy has a computational complexity of $\mathcal O(ncsm^2)$, where the $c$ denotes the number of clusters and the maximum iteration number is   $n$.   Solving the other sub-problem   $\vV_2$ takes $\mathcal O(csm^2)$  using the Germs algorithm. Thus, each iteration has a total computational cost of   $\mathcal O(ncsm^2) +\mathcal O(csm^2) = \mathcal O((n+1)ncsm^2)$. In total, the complexity to learn the uniform embedding in Algorithm~\ref{ALg_1} is $\mathcal O(mk^4+ K((n+1)ncsm^2))$, with  $K$ being the number of iterations.
		
		\section{Conclusion}
		The clustering of small-size samples with large dimensions is a bottleneck problem. We propose a unified learning model to effectively fuse multiple samples' affinities of different orders to obtain the sample's  uniform low-dimensional embedding. Tensor-vector products uniformly formulate the sample affinities of various orders, which opens a new door for studying multiple samples. Our method involves three matrix products: arithmetical, Khatri-Rao and Kronecker product. They are jointly optimized to yield a uniform low-dimensional embedding of the samples. Experiments on synthetic data demonstrate the power of fusing different-order affinities and experiments on several real datasets with small sample size yet large feature dimensionality show the effectiveness and superiority of the method over other popular approaches.
	\ifCLASSOPTIONcaptionsoff
  	\newpage
	\fi
	\bibliographystyle{IEEEtran}
	\bibliography{reference}

	\end{document}